\documentclass[lettersize,journal]{IEEEtran}
\usepackage{amsmath,amsfonts}
\usepackage{booktabs}
\usepackage{cite}
\usepackage{array}
\usepackage{textcomp}
\usepackage{amsthm}
\usepackage[english]{babel}
\newtheorem{theorem}{Theorem}
\newtheorem{lemma}{Lemma}
\newtheorem{remark}{Remark}
\newtheorem{definition}{Definition}
\newtheorem{assumption}{Assumption}
\usepackage{stfloats}
\usepackage{graphicx}
\usepackage[caption=false,font=footnotesize]{subfig}

\usepackage{url}
\usepackage{verbatim}
\usepackage{graphicx}
\usepackage{amssymb}
\usepackage{algorithm2e}
\usepackage{amsmath}
\allowdisplaybreaks
\RestyleAlgo{ruled}
\def\BibTeX{{\rm B\kern-.05em{\sc i\kern-.025em b}\kern-.08em
    T\kern-.1667em\lower.7ex\hbox{E}\kern-.125emX}}
\usepackage{balance}
\usepackage{xcolor}
\usepackage{algpseudocode}
\usepackage{ifthen}
\usepackage{xcolor} % For defining custom colors
\usepackage{soul}   % For highlighting

\usepackage{multirow} % Add this line in your preamble
\definecolor{lightblue}{RGB}{173,216,230} % Define custom color
\sethlcolor{lightblue}                    % Set highlight color

\begin{document}
\title{Coded Robust Aggregation for Distributed Learning under Byzantine Attacks}
\author{Chengxi Li,~\IEEEmembership{Member,~IEEE,} Ming Xiao,~\IEEEmembership{Senior Member,~IEEE,} and Mikael Skoglund,~\IEEEmembership{Fellow, IEEE}
\thanks{C. Li, M. Xiao and M. Skoglund are with the Division of Information Science and Engineering, School of Electrical Engineering and Computer Science, KTH Royal Institute of Technology, 10044 Stockholm, Sweden. (e-mail: chengxli@kth.se; mingx@kth.se; skoglund@kth.se). Corresponding author: Chengxi Li.}
}

\markboth{Journal of \LaTeX\ Class Files}%
{}

\maketitle

\begin{abstract}
In this paper, we investigate the problem of distributed learning (DL) in the presence of Byzantine attacks. For this problem, various robust bounded aggregation (RBA) rules have been proposed at the central server to mitigate the impact of Byzantine attacks. However, current DL methods apply RBA rules for the local gradients from the honest devices and the disruptive information from Byzantine devices, and the learning performance degrades significantly when the local gradients of different devices vary considerably from each other. To overcome this limitation, we propose a new DL method to cope with Byzantine attacks based on coded robust aggregation (CRA-DL). Before training begins, the training data are allocated to the devices redundantly. During training, in each iteration, the honest devices transmit coded gradients to the server computed from the allocated training data, and the server then aggregates the information received from both honest and Byzantine devices using RBA rules. In this way, the global gradient can be approximately recovered at the server to update the global model. Compared with current DL methods applying RBA rules, the improvement of CRA-DL is attributed to the fact that the coded gradients sent by the honest devices are closer to each other. This closeness enhances the robustness of the aggregation against Byzantine attacks, since Byzantine messages tend to be significantly different from those of honest devices in this case. We theoretically analyze the convergence performance of CRA-DL. Finally, we present numerical results to verify the superiority of the proposed method over existing baselines, showing its enhanced learning performance under Byzantine attacks.
\end{abstract}


\section{Introduction}
\label{introduction} 
\IEEEPARstart{D}{istributed} learning (DL) has recently attracted significant attention \cite{chen2021distributed, verbraeken2020survey}. 
 In DL, the central server acts as a central processor with access to a very large dataset. To accelerate training on this large dataset, the server distributes the computational workload to multiple devices, which function as worker nodes. This is a common practice in distributed computing for machine learning applications.
Compared to training on a single device, DL leverages the computational resources of various edge devices, thereby increasing training efficiency \cite{shuvo2022efficient, wang2020convergence}. Typically, the training process of DL involves multiple iterations. In each iteration, the server first sends the global model to the devices. After receiving the global model, each device computes the local gradient based on its local dataset and transmits the local gradient to the server. The server aggregates the local gradients from all devices to obtain the global gradient to update the global model \cite{abdulrahman2020survey}.

Due to the distributed nature of the DL system, it is unrealistic for the server to continuously monitor the operational status of the devices to ensure they function properly at all times. Issues such as computation errors, crashes, and stalled processes may arise during training \cite{blanchard2017machine}. Besides, some external attackers may compromise devices before the deployment of the system by injecting malicious firmware or backdoors through the hardware or software supply chain, wherein vendors or their code are  modified with malicious purposes \cite{alkhadra2021solar}. These malfunctioning or compromised devices send incorrect messages during the training, which are known as Byzantine devices \cite{yang2020adversary,zhang2024nspfl,zhang2022lsfl,abrardo2016game, liu2020optimal,li2025sign}. DL systems affected by them are said to be under Byzantine attacks. To deal with Byzantine attacks, current DL approaches can be classified into two categories as follows. 

In the first category, various aggregation rules are designed at the server, which are robust to the Byzantine attacks. For instance, in \cite{yin2018byzantine}, coordinate-wise median and trimmed mean are adopted to aggregate the information from the devices, and the error rates for strongly convex, non-strongly convex, and smooth non-convex functions are analyzed. In \cite{karimireddy2021learning}, a robust iterative clipping aggregation rule is proposed, where momentum is incorporated to deal with time-coupled Byzantine attacks. In \cite{chen2017distributed}, the geometric median is used to aggregate the messages from the devices, resulting in a variant of the typical gradient descent method. In \cite{xia2019faba}, a fast aggregation method is proposed that removes outliers in the messages uploaded by the devices to obtain local gradients closer to the true ones. In \cite{xie2018phocas}, a trimmed mean-based approach that is also dimensionally Byzantine-resilient is proposed, which is demonstrated to have nearly linear time complexity. In \cite{blanchard2017machine}, the messages from the devices are aggregated using majority-based and squared-distance-based methods, where the vectors that minimize the distances to their closest vectors are selected as trustworthy.  In \cite{zhao2024huber}, a new aggregator at the server is designed based on minimization of Huber loss, which attains enhanced robustness under a certain ratio of Byzantine devices under the independent and identically distributed (i.i.d) assumption. In \cite{luan2024robust}, maximum correntropy aggregation (MCA) is proposed by using the maximum correntropy criterion to obtain the central value among all messages. 
In \cite{dong2023byzantine}, it is shown that most of the above state-of-the-art robust aggregation rules are all robust bounded aggregation (RBA) rules, where the bias between the aggregation result and the average of the messages from the honest devices is bounded by the the largest deviation of the messages from the honest devices. Although current DL methods with RBA rules at the server can achieve satisfactory Byzantine resilience under certain conditions, they have a significant shortcoming: degradation of learning performance when the local gradients of different devices vary considerably. This degradation occurs because the input to the RBA rules includes both the local gradients from honest devices and the disruptive information from Byzantine devices. When the local gradients of different devices vary significantly due to heterogeneity among subsets of the training data, the disruptive information from the Byzantine devices can more easily mislead and manipulate the output of the aggregator.

 In the second category, gradient coding methods are developed, where the training dataset is divided into subsets, and these subsets are redundantly assigned to devices before training. In this way, each device obtains multiple subsets. Leveraging this redundancy, in each training iteration, each device computes local gradients corresponding to its assigned subsets and encodes them into a coded gradient. Honest devices transmit the coded gradients to the server, while Byzantine devices send incorrect messages. Based on the received messages from all devices, the server can fully identify the Byzantine devices, enabling the accurate recovery of the true global gradient as if no Byzantine devices were present. For instance, in \cite{hong2024group}, coding techniques are combined with group-wise verification to deal with Byzantine attacks, specifically designed for matrix multiplication tasks in DL. In \cite{data2020data}, a method specifically designed for matrix-vector multiplication in DL is proposed based on error correction with real numbers under Byzantine attacks, which is proved to be information-theoretically optimal with deterministic guarantees. For more general DL problems, in \cite{hofmeister2024byzantine}, based on fractional repetition allocation of the training data, erroneous messages from the Byzantine devices are detected and transformed into erasures at the cost of additional local computations of the server and additional communication in each iteration. In \cite{chen2018draco}, the encoding of the local gradients is designed by using the fractional repetition code and cyclic repetition code, and the decoders are proposed based on majority vote and Fourier technique. However, these gradient coding techniques still require a very high level of redundancy in the allocation of training data among the devices to fully recover the true global gradient in each iteration, leading to significant computation and storage burdens on the devices. 

In addition to dealing with Byzantine devices in DL, gradient coding techniques have also been explored to address other problems, such as the non-responsive devices in DL commonly referred to as stragglers \cite{tandon2017gradient,ozfatura2019gradient,buyukates2022gradient,glasgow2021approximate, wang2019erasurehead, bitar2020stochastic, li2024distributed}. Before training begins, the subsets of the training data are allocated redundantly across devices. During training, non-straggler devices transmit coded gradients to the server in each iteration based on the local gradients computed from the local subsets, while stragglers do not transmit anything. The server can decode and recover the true global gradient using the received coded gradients from the non-stragglers. Depending on whether the true global gradient is recovered exactly or approximately, current gradient coding techniques for handling stragglers can be classified into exact gradient coding techniques \cite{tandon2017gradient,ozfatura2019gradient,buyukates2022gradient} and approximate gradient coding techniques \cite{glasgow2021approximate, wang2019erasurehead, bitar2020stochastic, li2024distributed}. Given that machine learning algorithms are inherently robust to noise, it may not be necessary to fully recover the true global gradient in each iteration. Obtaining an approximate version of the true global gradient may suffice for training machine learning models. As a result, approximate gradient coding techniques have gained significant attention recently, which require only a modest level of redundancy in the allocation of training data and induce lower computational and storage burdens on the devices. Among the existing approximate gradient coding techniques, stochastic gradient coding (SGC), originally proposed in \cite{bitar2020stochastic}, requires very simple encoding and decoding techniques while achieving satisfactory learning performance. In SGC, the subsets of the training data are allocated to devices in a pair-wise balanced manner, and this method has been applied in various DL scenarios with stragglers \cite{li2024distributed,li2024gradient}. Although SGC was originally proposed to cope with the stragglers in DL, its advantages could also be leveraged to combat Byzantine attacks in DL. Nonetheless, leveraging the strengths of SGC to improve Byzantine resilience remains a significant challenge and is yet to be thoroughly investigated. 

To overcome the shortcomings of existing techniques designed for DL under Byzantine attacks, we propose a new DL method based on coded robust aggregation (CRA-DL). In CRA-DL, before training begins, the training data are divided into subsets and allocated redundantly to the devices in a pair-wise balanced manner, motivated by the SGC scheme. In each training iteration, the server transmits the global model to all devices, and each device computes the local gradients based on its local training data subsets. Subsequently, the local gradients corresponding to different subsets are encoded to generate a single vector on each device, known as the coded gradient. Each honest device transmits its coded gradient to the server, while the Byzantine devices send arbitrarily incorrect messages. The server then receives the vectors from all devices and applies an RBA rule to these messages, which yields a final global update that approximates the global gradient. Finally, the global model is updated at the server with the global update. Note that, this is a meta algorithm that can be employed with any RBA rules proposed in the literature. 
%Compared with current DL methods applying RBA rules, the intuition behind the superiority of the proposed method is as follows: the coded gradients sent by the honest devices are much closer to each other compared with the local gradients sent by the honest devices in existing methods, benefiting from the redundancy in data allocation among the devices, which enhances the robustness of the aggregation at the server under Byzantine attacks. 
We analyze the convergence performance of CRA-DL. Additionally, we present ample numerical results to verify that CRA-DL outperforms existing baselines. Our contributions are listed as follows: 
\begin{enumerate}
    \item We propose a new method, i.e., CRA-DL, to deal with the DL problem under Byzantine attacks by simultaneously exploiting the advantages of SGC and RBA rules. To the best of our knowledge, this is the first work to integrate SGC with RBA in the context of Byzantine-resilient DL. This integration is not a trivial combination, but rather a principled design that leverages data allocation redundancy to bring coded gradients from honest devices closer together. This, in turn, amplifies the effectiveness and robustness of the aggregation at the server and leads to improved learning performance.
    \item We analyze the convergence performance of CRA-DL for non-convex loss functions and show that the convergence rate improves and the asymptotic solution error decreases as the redundancy in data allocation increases and the robustness of the RBA rules increases. This offers an insight that, to the best of our knowledge, has not been explored in existing literature.
    \item  Our numerical results verify that CRA-DL significantly improves learning performance and enhances robustness to Byzantine attacks in DL in various scenarios. 
\end{enumerate}

Based on the novelty of our work, the main differences between this paper and existing literature are summarized as follows: 
    \begin{enumerate}
    \item Compared to existing methods that directly apply RBA rules at the server to aggregate local gradients \cite{yin2018byzantine,karimireddy2021learning,chen2017distributed,xia2019faba,xie2018phocas,blanchard2017machine,zhao2024huber,dong2023byzantine}, the proposed method aggregates coded vectors at the server using RBA rules to generate the global model update. Since the coded vectors from honest devices are more similar, the robustness of the aggregation at the server is enhanced, thereby improving learning performance.
    
    \item In traditional gradient coding approaches dealing with Byzantine attacks \cite{hong2024group,data2020data,hofmeister2024byzantine,chen2018draco}, the primary objective is to fully identify Byzantine devices at the server in order to accurately recover the true global gradient as if no Byzantine devices were present. In contrast, the key idea in our method is to apply RBA rules at the server to aggregate coded gradients, allowing for an approximate recovery of the global gradient, which is then used to update the global model. Compared to current gradient coding methods, the proposed method attains Byzantine robustness with a significantly lower level of redundancy in the allocation of training data among devices. In other words, the proposed method imposes lower computational and storage burdens on the devices.
 \item In SGC, the original goal is to address the straggler problem in DL, rather than to defend against Byzantine attacks. Based on redundancy in data allocation, coded gradients from the non-straggler devices allow the missing information to be recovered at the server. However, directly applying SGC does not provide satisfactory resilience against Byzantine attacks. This is because, under Byzantine attacks, the challenge is not merely missing information but the presence of incorrect or malicious messages from Byzantine devices, which can mislead the aggregation process. How to leverage and aggregate coded gradients, using the same redundancy in data allocation and encoding principles as SGC, to enhance Byzantine resilience is a fundamentally different problem from using SGC solely to mitigate stragglers. This remains an open challenge in the existing literature. In this work, although we adopt the same redundancy in data allocation and encoding principles as SGC, we innovatively apply RBA rules at the server to aggregate the coded gradients and thus enhance Byzantine robustness, instead of relying on the original aggregation rule used in SGC.
    
%    \item Due to the different implementations of the proposed method compared to existing approaches, the theoretical analysis in this paper differs significantly from those in the literature. Specifically, this paper provides a rigorous convergence analysis for the proposed method under both fixed and decaying learning rates.
\end{enumerate}

The structure of this paper is as follows. In Section \ref{problem model}, we introduce the problem model. In Section \ref{our method}, we propose our method and describe the implementation procedure. In Section \ref{performance analysis}, we present the performance analysis from a theoretical perspective. The numerical results are shown in Section \ref{simulations} to demonstrate the superior performance of the proposed method. Finally, concluding remarks are provided in Section \ref{conclusions}. 

\section{Problem Model}
\label{problem model}
The considered problem is introduced as follows. There are $N$ devices and a central server in the DL system, whose goal is to train a model by solving the optimization problem \cite{chen2021distributed, liu2022distributed,chen2018draco}: 
\begin{align}
    \label{basic problem}
   {{\mathbf{x}}^*} = \arg {\min _{{\mathbf{x}} \in {\mathbb{R}^D}}}F\left( {\mathbf{x}} \right), 
\end{align}
where ${\mathbf{x}}$ represents the model parameter vector, and $F\left( {\mathbf{x}} \right)$ is the overall training loss defined as
\begin{align}
\label{overall loss}
F({\mathbf{x}}) = \sum\limits_{\varrho  \in \mathcal{D}} {l\left( {{\mathbf{x}},\varrho } \right)} ,
\end{align}
where \({l}\left( {{\mathbf{x}},{\varrho}} \right): \mathbb{R}^D \to \mathbb{R}\) represents the training loss based on the training data sample \({{\varrho}}\) in the training dataset \({\mathcal{D}}\). Without specification, all vectors in this paper are column vectors. 

Under the typical DL framework\cite{mishchenko2024distributed}, before the training starts, the training dataset \({\mathcal{D}}\) is divided into $N$ non-overlapping subsets and allocated to $N$ devices so that each device obtains one subset. During the training, in iteration $t$, the current global model ${\mathbf{x}}^t$ is transmitted from the server to the devices. Then, each device computes the local gradient corresponding to its subset and sends the local gradient to the server. After receiving the local gradients from the devices, the server aggregates them to form the global gradient to update the global model and to obtain ${\mathbf{x}}^{t+1}$ \cite{chen2018draco}.  

In the above system, some devices may be under Byzantine attacks due to malicious attacks or malfunction of the devices \cite{yang2020adversary,zhang2024nspfl,zhang2022lsfl,abrardo2016game, liu2020optimal}. Let us define \({\mathcal{B}^t}\) as the set containing all indices of the Byzantine devices in iteration $t$, and use \({\mathcal{H}^t}\) as the set containing all indices of the honest devices in iteration $t$. During each iteration, the honest devices transmit truth-worthy messages to the server as expected, while the Byzantine devices transmit arbitrarily incorrect information to the server. Without prior knowledge of the Byzantine devices, it is assumed that in each iteration, a certain fraction \(\alpha\) of the devices are Byzantine, while the rest are honest \cite{dong2023byzantine}. The server does not know the identities of the devices beforehand and only knows the value of $\alpha$ \cite{dong2023byzantine}. The identities of the devices are considered to be random and independent across iterations \cite{chen2018draco}, which implies that the identities of the devices are non-persistent across iterations. If the identities were persistent, the server could potentially identify the Byzantine devices by accumulating information over time. From this perspective, the non-persistent case represents the strongest form of Byzantine attacks.   

For the above problem, our aim is to enhance the learning performance under Byzantine attacks.

\section{The Proposed Method: CRA-DL}
\label{our method}

In this section, we describe the implementation details of the proposed CRA-DL method. 

Before the training starts, the training dataset \(\mathcal{D}\) is divided into \(M\) subsets, represented as \(\mathcal{D} = \left\{ \mathcal{D}_1, \dots, \mathcal{D}_M \right\}\). These subsets are allocated to \(N\) devices in a pair-wise balanced manner, motivated by the advantages of the SGC scheme \cite{bitar2020stochastic}. Specifically, each device \(i\) holds \(r\) subsets from the total set of subsets\footnote{We focus on the case where the same number of subsets is allocated to each device. It is worth noting that the analysis in this paper can be easily extended to the case where different numbers of subsets are assigned to each device.}. The number of subsets held by both device \(i\) and device \(j\) is \(\frac{r^2}{M}\), for \(i \neq j\). Let us denote the number of devices that hold subset \(\mathcal{D}_k\) as \(d_k\), \(\forall k\). We define a data allocation matrix \(\mathbf{S}\), where \(s(i,k)\) is the \((i,k)\)-th element. If \(s(i,k) = 1\), subset \(\mathcal{D}_k\) is allocated to device \(i\); otherwise, it is not. Based on the above setting, the problem in (\ref{basic problem}) can be equivalently expressed as:
\begin{align}
    \label{basic problem 2}
   {{\mathbf{x}}^*} = \arg {\min _{{\mathbf{x}} \in {\mathbb{R}^D}}}F\left( {\mathbf{x}} \right) \triangleq \arg {\min _{{\mathbf{x}} \in {\mathbb{R}^D}}}\sum\limits_{k = 1}^M {{f_k}\left( {\mathbf{x}} \right)}, 
\end{align}
where \(f_k({\mathbf{x}}): \mathbb{R}^D \to \mathbb{R}\) denotes the training loss associated with subset \({\mathcal{D}_k}\):
\begin{align}
\label{fix}
{f_k}({\mathbf{x}}) = \sum\limits_{\varrho  \in {\mathcal{D}_k}} {l\left( {{\mathbf{x}},\varrho } \right)}.
\end{align}

Next, during the training process, in iteration \(t\), the server sends the current global model \(\mathbf{x}^t\) to all devices. After that, each device $i$ computes the local gradients associated with its local subsets and obtains $\left\{ {\nabla {f_k}\left( {{{\mathbf{x}}^t}} \right)|k\in \left\{ {1,...,M} \right\},s\left( {i,k} \right) \ne 0} \right\}$. Based on that, device $i$ encodes the local gradients into a single vector as
\begin{align}
    \label{encoding}
    {\mathbf{g}}_i^t = \sum\limits_{k \in \left\{ {\left. k \right|s\left( {i,k} \right) \ne 0} \right\}} {\frac{1}{{{d_k}}}} \nabla {f_k}\left( {{{\mathbf{x}}^t}} \right).
\end{align}
The operation of linear combination in (\ref{encoding}) is referred to as coding, as in much of the existing literature related to gradient coding \cite{tandon2017gradient, hofmeister2024byzantine}, because this linear combination is designed so that if the server receives incorrect messages from a subset of devices, it can still approximately recover the true full gradient, as will be shown later in this paper. This process resembles how error-correcting codes recover a message even when some symbols are corrupted.

If device $i$ is honest, i.e., \(i \in {\mathcal{H}^t}\), the coded gradient ${\mathbf{g}}_i^t$ is sent to the server, $\forall i$. If device \(j\) is a Byzantine device, i.e., \(j \in \mathcal{B}^t\), it transmits incorrect information to the server, denoted by \(\mathbf{b}_j^t\), which is the same size as \(\mathbf{g}_j^t\) but contains different elements. With the received messages from the devices, denoted by \(\{{{\left\{ {{\mathbf{g}}_i^t} \right\}}_{i \in {\mathcal{H}^t}}},{{\left\{ {{\mathbf{b}}_j^t} \right\}}_{j \in {\mathcal{B}^t}}}\}\), the server adopts an RBA rule\footnote{{There are also other definitions of robust aggregation rules, such as those in \cite{karimireddy2021learning} and \cite{wu2023byzantine}. In this paper, however, we focus on RBA rules, as many state-of-the-art robust aggregation rules fall within their scope.}} $A(\cdot)$ to yield the global model update as 
\begin{align}
    \label{global model update}
    {{\mathbf{\hat g}}^t} = A\left( {{{\left\{ {{\mathbf{g}}_i^t} \right\}}_{i \in {\mathcal{H}^t}}},{{\left\{ {{\mathbf{b}}_j^t} \right\}}_{j \in {\mathcal{B}^t}}}} \right),
\end{align}
which is an approximate version of the global gradient. Here, RBA rules are utilized, considering that various state-of-the-art robust aggregation rules with recent advances fall within their scope \cite{dong2023byzantine}. 
For example, the coordinate-wise median belongs to the class of RBA rules, and its specific expression is given below. Given $N$ vectors $\mathbf{z}_1, \ldots, \mathbf{z}_N \in \mathbb{R}^D$, the coordinate-wise median is defined as:
\begin{align}
\label{median}
    &\mathrm{Median}(\mathbf{z}_1, \ldots, \mathbf{z}_N) = \nonumber \\
&\left[ 
\mathrm{median}\big( z_{1,1}, \ldots, z_{N,1} \big), \;
\ldots, \;
\mathrm{median}\big( z_{1,D}, \ldots, z_{N,D} \big)
\right]^\top,
\end{align}
where $z_{n,d}$ denotes the $d$-th coordinate of the $n$-th vector.
RBA rules have been well-defined in \cite{dong2023byzantine}, and the definition is provided below. 
\begin{definition}[RBA rules \cite{dong2023byzantine}]
\label{def_robust agg}
Suppose there are $N_1$ messages ${{\mathbf{z}}_1},...,{{\mathbf{z}}_{{N_1}}} \in {\mathbb{R}^D}$ from $N_1$ honest devices and $N_2$ messages ${{{\mathbf{\tilde z}}}_1},...,{{{\mathbf{\tilde z}}}_{{N_2}}} \in {\mathbb{R}^D}$ from $N_2$ Byzantine devices. The fraction of Byzantine devices is $\alpha=\frac{N_2}{N_1+N_2}$. An aggregation rule $A(\cdot)$ is an RBA rule, if the difference between the aggregation result and the average of the messages from the honest devices is bounded by
\begin{align}
    \label{A robust bound}
    {\left\| {A\left( {{{\left\{ {{{\mathbf{z}}_i}} \right\}}_{i \in \left\{ {1,...,{N_1}} \right\}}},{{\left\{ {{{{\mathbf{\tilde z}}}_j}} \right\}}_{j \in \left\{ {1,...,{N_2}} \right\}}}} \right) - {\mathbf{\bar z}}} \right\|^2} \leqslant C_\alpha ^2{\varsigma},
\end{align}
where ${\mathbf{\bar z}}$ is the average of the messages from the honest devices denoted by ${\mathbf{\bar z}}= \frac{1}{{{N_1}}}\sum\limits_{i = 1}^{{N_1}} {{{\mathbf{z}}_i}}$, $\varsigma$ is defined as $\varsigma  = {\max _{i \in \left\{ {1,...,{N_1}} \right\}}}{\left\| {{\mathbf{\bar z}} - {{\mathbf{z}}_i}} \right\|^2}$, and $C_\alpha^2$ is a constant determined by the value of $\alpha$. The values of \(C_\alpha^2\) for some commonly used RBA rules are provided in Table I \cite{dong2023byzantine}, where $N=N_1+N_2$. From (\ref{A robust bound}), it can be seen that a more accurate aggregation result can be achieved when the messages sent by the honest devices are closer to each other.
\end{definition}
Note that the proposed method is a meta-algorithm that can be adapted based on the choice of any particular RBA rule. Consequently, various state-of-the-art RBA rules, including those introduced in \cite{yin2018byzantine,karimireddy2021learning,chen2017distributed,xia2019faba,xie2018phocas,blanchard2017machine,dong2023byzantine}, can be applied within our proposed method. 
\begin{table}
\centering
\caption{The values of $C_\alpha^2$ of some commonly used RBA rules \cite{dong2023byzantine}}
\begin{tabular}{ll}
\toprule
\textbf{RBA Rules} & \textbf{$C_\alpha^2$} \\
\midrule
Coordinate-wise median \cite{yin2018byzantine} & $\frac{1}{{2{{\left( {1 - \alpha } \right)}^2}}}{\left[ {\min \left\{ {2\sqrt {N - N\alpha } ,\sqrt D } \right\}} \right]^2}$ \\
Trimmed mean \cite{yin2018byzantine} & $\frac{{2\alpha \left( {1 - \alpha } \right)}}{{{{\left( {1 - 2\alpha } \right)}^2}}}$ \\
Geometric median \cite{chen2017distributed} & ${\left[ {\frac{{2\left( {1 - \alpha } \right)}}{{1 - 2\alpha }}} \right]^2}$ \\
Krum \cite{blanchard2017machine} & $2{\left( {1 + \sqrt {\frac{{1 - \alpha }}{{1 - 2\alpha }}} } \right)^2}$ \\
Phocas \cite{xie2018phocas} & $4 + \frac{{12\alpha \left( {1 - \alpha } \right)}}{{{{\left( {1 - 2\alpha } \right)}^2}}}$ \\
FABA \cite{xia2019faba} & $4\left( {\frac{{N\alpha }}{{N - N\alpha }} + \frac{{N + 1 - N\alpha }}{{N - N\alpha }}\frac{{N\alpha }}{{N - 3N\alpha }}} \right)$ \\
\bottomrule
\end{tabular}
\end{table}
%In the following, an exemplary RBA rule, i.e., coordinate-wise median, is introduced. The aggregation result of coordinate-wise median can be written as 
% \begin{align}
%     \label{median}
%     &A_{\text{CWM}}\left( {{{\left\{ {{{\mathbf{z}}_i}} \right\}}_{i \in \left\{ {1,...,{N_1}} \right\}}},{{\left\{ {{{{\mathbf{\tilde z}}}_j}} \right\}}_{j \in \left\{ {1,...,{N_2}} \right\}}}} \right) \nonumber\\
%     =& \left[ \begin{gathered}
%   {\text{med}}\left( {{{\left\{ {{z_{i,1}}} \right\}}_{i \in \left\{ {1,...,{N_1}} \right\}}},{{\left\{ {{{\tilde z}_{j,1}}} \right\}}_{j \in \left\{ {1,...,{N_2}} \right\}}}} \right) \hfill \\
%   ... \hfill \\
%   {\text{med}}\left( {{{\left\{ {{z_{i,D}}} \right\}}_{i \in \left\{ {1,...,{N_1}} \right\}}},{{\left\{ {{{\tilde z}_{j,D}}} \right\}}_{j \in \left\{ {1,...,{N_2}} \right\}}}} \right) \hfill \\ 
% \end{gathered}  \right],
% \end{align}
% where \({\text{med}}\left( {\cdot} \right)\) denotes the median function, ${{z_{i,n}}}$ represents the $n$-th element in ${{{\mathbf{z}}_i}}$, and ${{{\tilde z}_{j,n}}}$ is the $n$-th element in ${{{{\mathbf{\tilde z}}}_j}}$, $\forall n$. The explicit forms of other RBA rules can be found in \cite{dong2023byzantine}. 

At the end of iteration $t$, the global model is updated at the server as
\begin{align}
    \label{update global model}
    {{\mathbf{x}}^{t + 1}} = {{\mathbf{x}}^t} - {\gamma ^t}{{\mathbf{\hat g}}^t},
\end{align}
where $\gamma ^t$ is the learning rate. The paradigm of the proposed method is shown as Fig.~\ref{fig: paradigm_method}, which is also presented as Algorithm~\ref{alg1}. 

\begin{algorithm}
\label{alg1}
\caption{CRA-DL}
\KwIn{Training data $\mathcal{D} = \{\mathcal{D}_1, \dots, \mathcal{D}_M\}$, learning rate $\{\gamma^t\}$.}
\KwOut{Trained model $\mathbf{x}^{T+1}$.}

\textbf{Initialization:} Initialize the model $\mathbf{x}^0$.

\For{$t = 0$ \KwTo $T$}{
    Server sends the global model $\mathbf{x}^t$ to all devices.

    \For{each device $i$ in parallel}{
        Compute local gradients: $\{\nabla f_k(\mathbf{x}^t) | s(i, k) = 1\}$. \\
        Encode local gradients as (\ref{encoding}). \\
        \eIf{$i \in \mathcal{H}^t$ (honest)}{
            Transmit $\mathbf{{g}}_i^t$ to the server.
        }{
            Transmit an incorrect vector $\mathbf{{b}}_i^t$ to the server.
        }
    }

    Server receives $\{\mathbf{{g}}_i^t\}_{i \in \mathcal{H}^t}$ and $\{\mathbf{{b}}_j^t\}_{j \in \mathcal{B}^t}$. \\
    Aggregate the messages using an RBA rule as (\ref{global model update}). \\
    Update the model as (\ref{update global model}).
}

\Return{$\mathbf{x}^{T+1}$}
\end{algorithm}

\begin{figure*}
    \centering
        \includegraphics[width=0.7\linewidth]{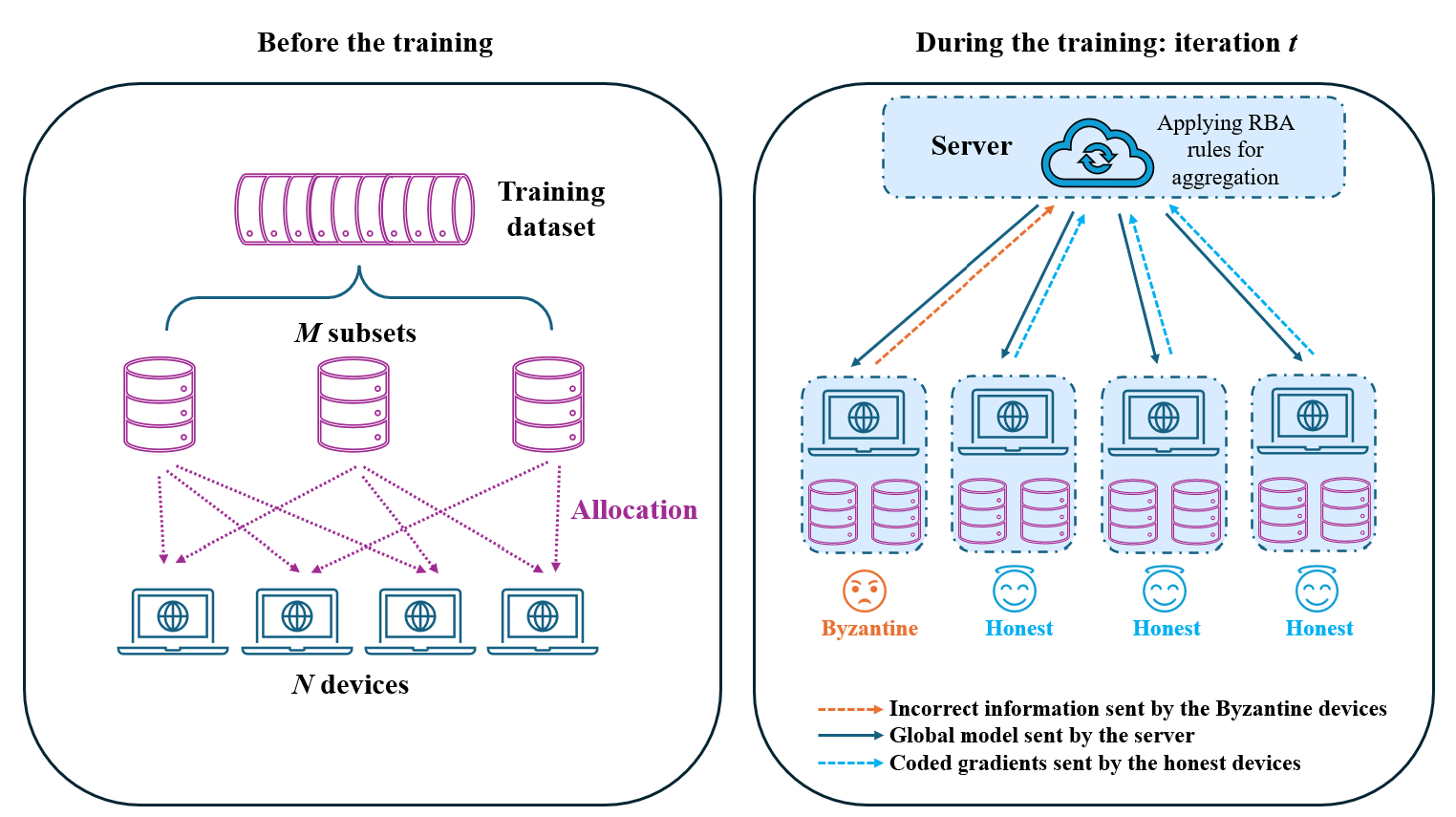}
    \caption{The paradigm of the proposed method.}
    \label{fig: paradigm_method}
\end{figure*}

    From a high-level perspective, compared with existing methods that apply RBA rules, the proposed method leverages redundancy in data allocation so that each honest device encodes its local gradients into a coded vector, rather than computing a single local gradient without redundancy. Based on that, the server aggregates these coded vectors together with potentially corrupted messages from Byzantine devices using RBA rules, rather than directly aggregating original local gradients alongside Byzantine messages, as done in conventional approaches. In this way, by leveraging the redundancy in data allocation, coded gradients are closer to each other compared to the original local gradients. By increasing the redundancy in data allocation, the coded gradients become increasingly similar. This will be analytically demonstrated in Section~\ref{performance analysis}. According to the properties of RBA rules implied by Definition~\ref{def_robust agg}, the difference between the aggregation output and the average of the messages from the honest devices is reduced by increasing the redundancy in data allocation. In this way, a more accurate global gradient can be recovered at the server under Byzantine attacks, which is then used to update the global model. This process improves learning performance and enhances the robustness against Byzantine attacks. 

\section{Performance Analysis}
\label{performance analysis}
In this section, we analyze the convergence performance of CRA-DL. First, let us state the assumptions, which have been widely used in the related field.  
\begin{assumption}
    \label{smooth assumption}
The overall training loss $F$ is $L$-smooth, which indicates the following inequality \cite{beznosikov2023biased, gorbunov2021marina}:
    \begin{align}
        \label{smooth assum}
  F\left( {\mathbf{x}} \right) \leq F\left( {\mathbf{y}} \right) + \left\langle {\nabla F\left( {\mathbf{y}} \right),{\mathbf{x}} - {\mathbf{y}}} \right\rangle  + \frac{L}{2}{\left\| {{\mathbf{x}} - {\mathbf{y}}} \right\|^2},\forall {\mathbf{x}},{\mathbf{y}}.
  \end{align}
\end{assumption}
\begin{assumption}
    \label{assp bounded heter}
    The heterogeneity among the subsets $\left\{ {{\mathcal{D}_1},...,{\mathcal{D}_M}} \right\}$ is bounded, indicating \cite{zhu2023byzantine} 
    \begin{align}
        \label{bounded heterogenity}
        \left\| {\nabla {f_i}\left( {\mathbf{x}} \right) - \frac{1}{M}\nabla F\left( {\mathbf{x}} \right)} \right\|^2 \leqslant {\beta ^2},\forall i,\forall {\mathbf{x}}.
    \end{align}
\end{assumption}
\begin{assumption}
    \label{lower bound}
    For some constant ${{F^*}}$, it holds that \cite{jin2024sign}
    \begin{align}
        \label{lower bound of F}
        F\left( {\mathbf{x}} \right) \geqslant {F^*},\forall \mathbf{x},
    \end{align}
which implies the overall training loss is lower bounded by $F^*$. 
\end{assumption}
Let us present two lemmas which aid the derivation of the main theorem.
\begin{lemma}
    \label{max distance}
    The maximum difference between any two coded gradients can be bounded as
    \begin{align}
        \label{lemma max distance}
 & {\max _{i,j \in \left\{ {1,...,N} \right\}}}{\left\| {{\mathbf{g}}_i^t - {\mathbf{g}}_j^t} \right\|^2} \nonumber \\
   \leqslant &8\frac{1}{{d_{\min }^2}}{\left( {r - \frac{{{r^2}}}{M}} \right)^2}\left( {{\beta ^2} + \frac{1}{{{M^2}}}{{\left\| {\nabla F\left( {\mathbf{x}^t} \right)} \right\|}^2}} \right),
    \end{align}
where \({d_{\min }} \triangleq \min \left\{ {{d_1},...,{d_M}} \right\}\).
\end{lemma}
\begin{proof}
    Please see Appendix~\ref{appendix lemma max distance}. 
\end{proof}
\begin{remark}
\label{remark_closer}
     When the values of \( d_1, \ldots, d_M \) do not vary significantly from each other\footnote{{In practice, this case is of particular importance, as fairness in DL is a major concern. This implies that each subset of the training data should be allocated to a similar number of devices, ensuring that all subsets contribute more equally to the trained model.}}, it holds that \( d_{\min}M \approx Nr \). Under this condition, we can rewrite (\ref{lemma max distance}) as
     \begin{align}
     \label{rew max distance}
         &{\max _{i,j \in \left\{ {1,...,N} \right\}}}{\left\| {{\mathbf{g}}_i^t - {\mathbf{g}}_j^t} \right\|^2}{\text{ }} \nonumber\\
         \leqslant & {\text{ }}8\frac{{{{\left( {M - r} \right)}^2}}}{{{N^2}}}\left( {{\beta ^2} + \frac{1}{{{M^2}}}{{\left\| {\nabla F\left( {{{\mathbf{x}}^t}} \right)} \right\|}^2}} \right).
     \end{align}
\end{remark}
From (\ref{rew max distance}), it can be observed that as the value of $r$ increases, meaning greater redundancy in data allocation, the maximum difference between any two coded gradients is bounded by a smaller value, which implies that the coded gradients are closer to each other.  As a special case, when $r = M$, the maximum difference between any two coded gradients becomes zero. In this case, all coded gradients sent by the honest devices are identical. In contrast, in existing methods based on RBA rules, there is no data allocation redundancy, and the maximum difference between any two local gradients sent by honest devices is determined by $\beta$, as defined in Assumption~\ref{assp bounded heter}. Without loss of generality, in the case where $M = N$, device $i$ transmits ${\nabla {f_i}\left( \mathbf{x}^t \right)}$ to the server, and device $j$ transmits ${\nabla {f_j}\left( \mathbf{x}^t \right)}$ to the server, for all $i, j$. In this case, the maximum difference between any two local gradients sent by the honest devices can be bounded as follows:
 \begin{align}
 \label{RBA existing dist}
 & {\max _{i,j}}{\left\| {\nabla {f_i}\left( {{{\mathbf{x}}^t}} \right) - \nabla {f_j}\left( {{{\mathbf{x}}^t}} \right)} \right\|^2} \nonumber \\
   \leqslant & {\max _{i,j}}{\left\| {\nabla {f_i}\left( {{{\mathbf{x}}^t}} \right) - \frac{1}{M}\nabla F\left( {{{\mathbf{x}}^t}} \right) + \frac{1}{M}\nabla F\left( {{{\mathbf{x}}^t}} \right) - \nabla {f_j}\left( {{{\mathbf{x}}^t}} \right)} \right\|^2} \nonumber \\
    \leqslant & 2{\max _i}{\left\| {\nabla {f_i}\left( {{{\mathbf{x}}^t}} \right) - \frac{1}{M}\nabla F\left( {{{\mathbf{x}}^t}} \right)} \right\|^2} \nonumber\\
    &+ 2{\max _j}{\left\| {\nabla{f_j}\left( {{{\mathbf{x}}^t}} \right) - \frac{1}{M}\nabla F\left( {{{\mathbf{x}}^t}} \right)} \right\|^2}
   \leqslant  4{\beta ^2}.
 \end{align}
 By comparing (\ref{rew max distance}) and (\ref{RBA existing dist}), it can be seen that the maximum difference among the messages sent by the honest devices can be reduced in the proposed method by increasing the level of redundancy in the training data allocation, potentially approaching zero. In contrast, this reduction is not possible in existing DL methods with RBA rules. Note from (\ref{A robust bound}) that the robustness of the aggregation at the server increases, and a more accurate estimate of the global gradient can be obtained when the messages sent by the honest devices are closer to each other. This applies to both the proposed method and existing DL methods that employ RBA rules. Based on that, the proposed method is more likely to achieve better learning performance compared with existing DL methods with RBA rules.

\begin{lemma}
    \label{honest average}
    Let us denote the average of the messages from the honest devices in iteration $t$ as
\begin{align}
    \label{honest average information}
    {{\mathbf{\bar g}}^t} = \frac{1}{{\left| {{\mathcal{H}^t}} \right|}}\sum\limits_{i \in {\mathcal{H}^t}} {{\mathbf{g}}_i^t}.
\end{align}
We can bound ${{\mathbf{\bar g}}^t}$ conditioned on the previous iterations as
\begin{align}
    \label{bound honest average}
    &\mathbb{E}\left( {\left. {{{\left\| {{{{\mathbf{\bar g}}}^t}} \right\|}^2}} \right|{\mathcal{F}^t}} \right)\nonumber\\
    \leqslant & \frac{{\left( {{\phi _1} - {\phi _2}} \right)2{r^2}}}{{{{\left( {1 - \alpha } \right)}^2}Nd_{\min }^2}}\left[ {{\beta ^2} + \frac{1}{{{M^2}}}{{{\left\| {\nabla F\left( {{{\mathbf{x}}^t}} \right)} \right\|}^2}}} \right]\nonumber\\
    &+ \frac{{{\phi _2}}}{{{{\left( {1 - \alpha } \right)}^2}{N^2}}}{{{\left\| {\nabla F\left( {{{\mathbf{x}}^t}} \right)} \right\|}^2}},
\end{align}
where $\mathbb{E}\left( {\left. {\cdot} \right|{\mathcal{F}^t}} \right)$ denotes the expectation conditioned on the previous iterations $0,...,t-1$, and $\phi_1$ and $\phi_2$ are defined as the following constants:
\begin{align}
    \label{phi1 and 2}
    \phi_1\triangleq1-\alpha,
    \phi_2\triangleq\frac{{\left( {1 - \alpha } \right)\left( {N - N\alpha  - 1} \right)}}{{N - 1}}.
\end{align}
\end{lemma}
\begin{proof}
    Please see Appendix~\ref{appendix lemma honest average}. 
\end{proof}

Next, based on Lemma~\ref{max distance} and Lemma~\ref{honest average}, we characterize the convergence performance of the proposed method in the following theorem.  
\begin{theorem}[Convergence performance of CRA-DL with fixed learning rates]
\label{convergence performance CRA}
Based on Assumptions~\ref{smooth assumption}-\ref{lower bound}, if ${C_\alpha } < \frac{{{d_{\min }}M}}{{2\sqrt 2 N\left( {r - \frac{{{r^2}}}{M}} \right)}}$, with fixed learning rates $\gamma^t =\gamma = \frac{\lambda }{{\sqrt {T + 1} }}$, $\lambda>0$, for $T > {\left( {\frac{{\lambda {\rho _2}}}{{{\rho _1}}}} \right)^2} - 1$, CRA-DL converges as
\begin{align}
\label{convergence fixed lr}
 &\frac{1}{{T + 1}}\sum\limits_{t = 0}^T {\mathbb{E}\left[ {{{\left\| {\nabla F\left( {{{\mathbf{x}}^t}} \right)} \right\|}^2}} \right]}\nonumber\\
  \leqslant & \frac{{F\left( {{{\mathbf{x}}^0}} \right) - F^*}}{{ {\lambda \sqrt {T + 1} {\rho _1} - {\lambda ^2}{\rho _2}}}} + \frac{{\sqrt {T + 1} {\rho _3} + \lambda {\rho _4}}}{{\sqrt {T + 1} {\rho _1} - \lambda {\rho _2}}},
\end{align}
where
\begin{align}
    \label{rho def1}
  {\rho _1}& \triangleq \frac{1}{N} - 2\left( {r - \frac{{{r^2}}}{M}} \right)\frac{{\sqrt {2C_\alpha ^2} }}{{{d_{\min }}M}} , \\
  \label{rho def2}
  {\rho _2}& \triangleq \frac{{\left( {{\phi _1} - {\phi _2}} \right)2{r^2}L}}{{{{\left( {1 - \alpha } \right)}^2}Nd_{\min }^2{M^2}}} + \frac{{{\phi _2}L}}{{{{\left( {1 - \alpha } \right)}^2}{N^2}}} \nonumber\\
  &+ 8\frac{{LC_\alpha ^2}}{{d_{\min }^2{M^2}}}{\left( {r - \frac{{{r^2}}}{M}} \right)^2},  \\
  \label{rho def3}
  {\rho _3}& \triangleq \frac{{{\beta ^2}M\sqrt {2C_\alpha ^2} }}{{{d_{\min }}}}\left( {r - \frac{{{r^2}}}{M}} \right), \\
  \label{rho def4}
  {\rho _4}& \triangleq \frac{{\left( {{\phi _1} - {\phi _2}} \right)2{r^2}{\beta ^2}L}}{{{{\left( {1 - \alpha } \right)}^2}Nd_{\min }^2}} + \frac{{8C_\alpha ^2{\beta ^2}L}}{{d_{\min }^2}}{\left( {r - \frac{{{r^2}}}{M}} \right)^2}.
\end{align}
\end{theorem}
\begin{proof}
    From Assumption 1, we can derive 
\begin{align}
    \label{smooth proof}
  &\mathbb{E}\left[ {\left. {F\left( {{{\mathbf{x}}^{t + 1}}} \right)} \right|{\mathcal{F}^t}} \right] \nonumber \\
   \leqslant& F\left( {{{\mathbf{x}}^t}} \right) + \mathbb{E}\left[ {\left. {\left\langle {\nabla F\left( {{{\mathbf{x}}^t}} \right),{{\mathbf{x}}^{t + 1}} - {{\mathbf{x}}^t}} \right\rangle } \right|{\mathcal{F}^t}} \right]\nonumber\\
   &+ \frac{L}{2}\mathbb{E}\left( {\left. {{{\left\| {{{\mathbf{x}}^{t + 1}} - {{\mathbf{x}}^t}} \right\|}^2}} \right|{\mathcal{F}^t}} \right) \nonumber \\
   \mathop  = \limits^{\left\langle 1 \right\rangle } & F\left( {{{\mathbf{x}}^t}} \right) - {\gamma ^t}\mathbb{E}\left[ {\left. {\left\langle {\nabla F\left( {{{\mathbf{x}}^t}} \right),{{{\mathbf{\hat g}}}^t}} \right\rangle } \right|{\mathcal{F}^t}} \right]+ \frac{L}{2}\mathbb{E}\left( {\left. {{{\left\| {{\gamma ^t}{{{\mathbf{\hat g}}}^t}} \right\|}^2}} \right|{\mathcal{F}^t}} \right) \nonumber \\
   =& F\left( {{{\mathbf{x}}^t}} \right) - {\gamma ^t}\mathbb{E}\left[ {\left. {\left\langle {\nabla F\left( {{{\mathbf{x}}^t}} \right),{{{\mathbf{\hat g}}}^t} - {{{\mathbf{\bar g}}}^t} + {{{\mathbf{\bar g}}}^t}} \right\rangle } \right|{\mathcal{F}^t}} \right] \nonumber\\
   &+ \frac{{L{{\left( {{\gamma ^t}} \right)}^2}}}{2}\mathbb{E}\left( {\left. {{{\left\| {{{{\mathbf{\hat g}}}^t} - {{{\mathbf{\bar g}}}^t} + {{{\mathbf{\bar g}}}^t}} \right\|}^2}} \right|{\mathcal{F}^t}} \right) \nonumber \\
   \mathop  \leqslant \limits^{\left\langle 2 \right\rangle }& F\left( {{{\mathbf{x}}^t}} \right) - {\gamma ^t}\mathbb{E}\left[ {\left. {\left\langle {\nabla F\left( {{{\mathbf{x}}^t}} \right),{{{\mathbf{\hat g}}}^t} - {{{\mathbf{\bar g}}}^t}} \right\rangle } \right|{\mathcal{F}^t}} \right] \nonumber\\
   &- {\gamma ^t}\mathbb{E}\left[ {\left. {\left\langle {\nabla F\left( {{{\mathbf{x}}^t}} \right),{{{\mathbf{\bar g}}}^t}} \right\rangle } \right|{\mathcal{F}^t}} \right] \nonumber \\
   &+ L{\left( {{\gamma ^t}} \right)^2}\mathbb{E}\left( {\left. {{{\left\| {{{{\mathbf{\hat g}}}^t} - {{{\mathbf{\bar g}}}^t}} \right\|}^2}} \right|{\mathcal{F}^t}} \right) + L{\left( {{\gamma ^t}} \right)^2}\mathbb{E}\left( {\left. {{{\left\| {{{{\mathbf{\bar g}}}^t}} \right\|}^2}} \right|{\mathcal{F}^t}} \right) \nonumber \\
  \mathop  \leqslant \limits^{\left\langle 3 \right\rangle }& F\left( {{{\mathbf{x}}^t}} \right) + {\gamma ^t}\frac{{\eta {{\left\| {\nabla F\left( {{{\mathbf{x}}^t}} \right)} \right\|}^2} + \frac{1}{\eta }\mathbb{E}\left[ {\left. {{{\left\| {{{{\mathbf{\hat g}}}^t} - {{{\mathbf{\bar g}}}^t}} \right\|}^2}} \right|{\mathcal{F}^t}} \right]}}{2}\nonumber\\
   &- {\gamma ^t}\frac{1}{N}{\left\| {\nabla F\left( {{{\mathbf{x}}^t}} \right)} \right\|^2} \nonumber \\
  & + L{\left( {{\gamma ^t}} \right)^2}\mathbb{E}\left( {\left. {{{\left\| {{{{\mathbf{\hat g}}}^t} - {{{\mathbf{\bar g}}}^t}} \right\|}^2}} \right|{\mathcal{F}^t}} \right) + L{\left( {{\gamma ^t}} \right)^2}\mathbb{E}\left( {\left. {{{\left\| {{{{\mathbf{\bar g}}}^t}} \right\|}^2}} \right|{\mathcal{F}^t}} \right),
\end{align}
$\forall \eta>0$, where $\left\langle 1 \right\rangle$ is obtained by substituting (\ref{update global model}) into (\ref{smooth proof}), $\left\langle 2 \right\rangle$ is derived from the basic inequality given as (\ref{basic ineq 1}), and $\left\langle 3 \right\rangle$ holds due to Young's Inequality and the following relationship:
\begin{align}
    \label{mean1}
    \mathbb{E}\left[ {\left. {{{{\mathbf{\bar g}}}^t}} \right|{\mathcal{F}^t}} \right] = \frac{1}{N}\nabla F\left( {{{\mathbf{x}}^t}} \right).
\end{align}
In (\ref{smooth proof}), we can bound \({\mathbb{E}\left[ {\left. {{{\left\| {{{{\mathbf{\hat g}}}^t} - {{{\mathbf{\bar g}}}^t}} \right\|}^2}} \right|{\mathcal{F}^t}} \right]}\) as
\begin{align}
    \label{bound 1}
    &{\mathbb{E}\left[ {\left. {{{\left\| {{{{\mathbf{\hat g}}}^t} - {{{\mathbf{\bar g}}}^t}} \right\|}^2}} \right|{\mathcal{F}^t}} \right]}\leqslant {C_\alpha ^2\mathbb{E}\left[ {\left. {{\varsigma ^t}} \right|{\mathcal{F}^t}} \right]}\nonumber\\
    &\leqslant C_\alpha ^2{\max _{i,j \in \left\{ {1,...,N} \right\}}}{\left\| {{\mathbf{g}}_i^t - {\mathbf{g}}_j^t} \right\|^2},
\end{align}
where ${\varsigma ^t} \triangleq {\max _{i \in {\mathcal{H}^t}}}{\left\| {{{{\mathbf{\bar g}}}^t} - {\mathbf{g}}_i^t} \right\|^2}$, the first inequality is obtained from (\ref{global model update}), (\ref{honest average information}) and Definition~\ref{def_robust agg}, and the second inequality is due to the inequality ${\max _{i \in {\mathcal{H}^t}}}{\left\| {{{{\mathbf{\bar g}}}^t} - {\mathbf{g}}_i^t} \right\|^2} \leqslant {\max _{i,j \in {\mathcal{H}^t}}}{\left\| {{\mathbf{g}}_i^t - {\mathbf{g}}_j^t} \right\|^2} \leqslant {\max _{i,j \in \left\{ {1,...,N} \right\}}}{\left\| {{\mathbf{g}}_i^t - {\mathbf{g}}_j^t} \right\|^2}$.
Substituting Lemma~\ref{max distance} into (\ref{bound 1}) yields
\begin{align}
    \label{bound 2}
  &\mathbb{E}\left[ {\left. {{{\left\| {{{{\mathbf{\hat g}}}^t} - {{{\mathbf{\bar g}}}^t}} \right\|}^2}} \right|{\mathcal{F}^t}} \right] \nonumber \\
   \leqslant& \frac{{8C_\alpha ^2}}{{d_{\min }^2}}{\left( {r - \frac{{{r^2}}}{M}} \right)^2}\left[ {{\beta ^2} + \frac{1}{{{M^2}}}{{{\left\| {\nabla F\left( {{{\mathbf{x}}^t}} \right)} \right\|}^2}}} \right].
\end{align}
After that, substituting (\ref{bound 2}) and (\ref{bound honest average}) in Lemma~\ref{honest average} into (\ref{smooth proof}), we have
\begin{align}
    \label{smooth 2}
 & {\gamma ^t}\left[ {\frac{1}{N} - \frac{{\eta  + \frac{{8C_\alpha ^2}}{{\eta {M^2}d_{\min }^2}}{{\left( {r - \frac{{{r^2}}}{M}} \right)}^2}}}{2}} \right]{\left\| {\nabla F\left( {{{\mathbf{x}}^t}} \right)} \right\|^2} \nonumber \\
 &  - L{\left( {{\gamma ^t}} \right)^2}{\left\| {\nabla F\left( {{{\mathbf{x}}^t}} \right)} \right\|^2}\left\{ {\frac{{\left( {{\phi _1} - {\phi _2}} \right)2{r^2}}}{{{{\left( {1 - \alpha } \right)}^2}Nd_{\min }^2{M^2}}} + \frac{{{\phi _2}}}{{{{\left( {1 - \alpha } \right)}^2}{N^2}}}} \right. \nonumber \\
 &  + \left. {C_\alpha ^28\frac{1}{{d_{\min }^2{M^2}}}{{\left( {r - \frac{{{r^2}}}{M}} \right)}^2}} \right\} \nonumber \\
   \leqslant & F\left( {{{\mathbf{x}}^t}} \right) - \mathbb{E}\left[ {\left. {F\left( {{{\mathbf{x}}^{t + 1}}} \right)} \right|{\mathcal{F}^t}} \right] + {\gamma ^t}\frac{1}{\eta }\frac{{4{\beta ^2}C_\alpha ^2}}{{d_{\min }^2}}{\left( {r - \frac{{{r^2}}}{M}} \right)^2} \nonumber \\
 &  + L{\left( {{\gamma ^t}} \right)^2}\left\{ {\frac{{\left( {{\phi _1} - {\phi _2}} \right)2{r^2}{\beta ^2}}}{{{{\left( {1 - \alpha } \right)}^2}Nd_{\min }^2}} + \frac{{8C_\alpha ^2{\beta ^2}}}{{d_{\min }^2}}{{\left( {r - \frac{{{r^2}}}{M}} \right)}^2}} \right\} .
\end{align}
In (\ref{smooth 2}), by setting \(\eta  = 2\left( {r - \frac{{{r^2}}}{M}} \right)\frac{{\sqrt {2C_\alpha ^2} }}{{{d_{\min }}M}}\), it holds that ${\frac{1}{N} - \frac{{\eta  + \frac{{8C_\alpha ^2}}{{\eta {M^2}d_{\min }^2}}{{\left( {r - \frac{{{r^2}}}{M}} \right)}^2}}}{2}}>0$ under the condition ${C_\alpha } < \frac{{{d_{\min }}M}}{{2\sqrt 2 N\left( {r - \frac{{{r^2}}}{M}} \right)}}$. Based on that, we can rewrite (\ref{smooth 2}) as 
\begin{align}
    \label{rho ineq}
  &{\gamma ^t}{\rho _1}{\left\| {\nabla F\left( {{{\mathbf{x}}^t}} \right)} \right\|^2} - {\left( {{\gamma ^t}} \right)^2}{\rho _2}{\left\| {\nabla F\left( {\mathbf{x}} \right)} \right\|^2} \nonumber \\
   \leqslant& F\left( {{{\mathbf{x}}^t}} \right) - \mathbb{E}\left[ {\left. {F\left( {{{\mathbf{x}}^{t + 1}}} \right)} \right|{\mathcal{F}^t}} \right] + {\gamma ^t}{\rho _3} + {\left( {{\gamma ^t}} \right)^2}{\rho _4},
\end{align}
where $\rho_1, \rho_2, \rho_3$ and $\rho_4$ are all positive constants defined as (\ref{rho def1})-(\ref{rho def4}).

Taking full expectation on both sides of (\ref{rho ineq}) yields
\begin{align}
    \label{full exp}
  {\gamma ^t}{\rho _1}\mathbb{E}\left( {{{\left\| {\nabla F\left( {{{\mathbf{x}}^t}} \right)} \right\|}^2}} \right) - {\left( {{\gamma ^t}} \right)^2}{\rho _2}\mathbb{E}\left( {{{\left\| {\nabla F\left( {\mathbf{x}} \right)} \right\|}^2}} \right) \nonumber \\
   \leqslant \mathbb{E}\left[ {F\left( {{{\mathbf{x}}^t}} \right)} \right] - \mathbb{E}\left[ {F\left( {{{\mathbf{x}}^{t + 1}}} \right)} \right] + {\gamma ^t}{\rho _3} + {\left( {{\gamma ^t}} \right)^2}{\rho _4}.
\end{align}
Rearranging the terms in (\ref{full exp}) and taking average over $T$ iterations, we can obtain
\begin{align}
    \label{average ite}
  &\frac{1}{{T + 1}}\sum\limits_{t = 0}^T {\left( {{\gamma ^t}{\rho _1} - {{\left( {{\gamma ^t}} \right)}^2}{\rho _2}} \right)\mathbb{E}\left[ {{{\left\| {\nabla F\left( {{{\mathbf{x}}^t}} \right)} \right\|}^2}} \right]}  \nonumber \\
   \leqslant& \frac{{F\left( {{{\mathbf{x}}^0}} \right) - \mathbb{E}\left[ {F\left( {{{\mathbf{x}}^{T + 1}}} \right)} \right]}}{{T + 1}} + \frac{1}{{T + 1}}\sum\limits_{t = 0}^T {\left[ {{\gamma ^t}{\rho _3} + {{\left( {{\gamma ^t}} \right)}^2}{\rho _4}} \right]}\nonumber\\
   \leqslant & \frac{{F\left( {{{\mathbf{x}}^0}} \right) - {F^*}}}{{T + 1}} + \frac{1}{{T + 1}}\sum\limits_{t = 0}^T {\left[ {{\gamma ^t}{\rho _3} + {{\left( {{\gamma ^t}} \right)}^2}{\rho _4}} \right]},
\end{align}
where Assumption 3 is applied to derive the last inequality. With fixed learning rates $\gamma^t =\gamma = \frac{\lambda }{{\sqrt {T + 1} }}$, for $T > {\left( {\frac{{\lambda {\rho _2}}}{{{\rho _1}}}} \right)^2} - 1$, we can rewrite (\ref{average ite}) as
\begin{align}
    \label{average ite2}
 &\frac{1}{{T + 1}}\sum\limits_{t = 0}^T {\mathbb{E}\left[ {{{\left\| {\nabla F\left( {{{\mathbf{x}}^t}} \right)} \right\|}^2}} \right]}\nonumber\\
  \leqslant & \frac{{F\left( {{{\mathbf{x}}^0}} \right) - F^*}}{{\left( {T + 1} \right)\left( {\gamma {\rho _1} - {\gamma ^2}{\rho _2}} \right)}} + \frac{{{\rho _3} + \gamma {\rho _4}}}{{{\rho _1} - \gamma {\rho _2}}}  \nonumber \\
   =& \frac{{F\left( {{{\mathbf{x}}^0}} \right) - F^*}}{{ {\lambda \sqrt {T + 1} {\rho _1} - {\lambda ^2}{\rho _2}} }} + \frac{{\sqrt {T + 1} {\rho _3} + \lambda {\rho _4}}}{{\sqrt {T + 1} {\rho _1} - \lambda {\rho _2}}},
\end{align}
which completes the proof. 
\end{proof}
\begin{remark}
\label{remark fix lr}
In Theorem~\ref{convergence performance CRA}, on the right-hand side of (\ref{convergence fixed lr}), the asymptotic solution error approaches 
\begin{align}
    \label{error term fix lr}
    \frac{{{\rho _3}}}{{{\rho _1}}}  \approx  \frac{{{\beta ^2}{M^2}\sqrt {2C_\alpha ^2} \left( {1 - \frac{r}{M}} \right)}}{{1 - 2\left( {1 - \frac{r}{M}} \right)\sqrt {2C_\alpha ^2} }},
\end{align}
as \( T \) approaches infinity, where the approximation holds based on \( d_{\min}M \approx Nr \) when the values of \( d_1, \ldots, d_M \) do not vary significantly from each other. From (\ref{error term fix lr}), as \( r \) increases, indicating a greater redundancy in the data allocation, the asymptotic solution error diminishes. In the special case where \( r = M \), the asymptotic solution error equals zero, and CRA-DL converges without solution error. This aligns with our intuition that when each device holds a copy of the entire training dataset, even in the presence of Byzantine devices, the true global gradient can be recovered at the server if the number of honest devices exceeds that of the Byzantine devices.
Moreover, as \( C_\alpha \) decreases, implying enhanced robustness of the RBA rules, the asymptotic solution error decreases as well. Since CRA-DL is a meta-algorithm that can be employed with any RBA rules, it is promising to achieve better learning performance when incorporating enhanced RBA rules. 
Besides, as shown in (\ref{error term fix lr}), when the heterogeneity among subsets is negligible, i.e., when the value of $\beta$ is very small, the solution error is also expected to be small. In the special case where $\beta = 0$, there is no solution error, and the proposed method is guaranteed to obtain the optimal model.

    Finally, based on (\ref{error term fix lr}), the average solution error over all training subsets, with respect to the expression 
\[
\frac{1}{T + 1} \sum_{t = 0}^T \mathbb{E}\left[ \left\| \frac{1}{M} \nabla F\left( \mathbf{x}^t \right) \right\|^2 \right],
\]
can be approximated as
\[
\frac{1}{M^2} \cdot \frac{\rho_3}{\rho_1} \approx \frac{\beta^2 \sqrt{2C_\alpha^2} \left(1 - \frac{r}{M} \right)}{1 - 2 \left(1 - \frac{r}{M} \right) \sqrt{2C_\alpha^2}}.
\]
From the above expression, it can be seen that the average solution error increases as $M$ increases, where $M$ reflects the potential growth in the total training data size. The rationale is as follows: when the number of devices $N$ is fixed and each device has a fixed computational burden, as indicated by the parameter $r$, the computational load per device remains constant, and the capacity of the system to process the data does not increase. In this case, as $M$ increases, the overall learning task becomes more demanding while the capacity of the system to process the growing data does not scale accordingly, leading to an increase in the average solution error.
In addition, we observe that the average solution error is determined by the ratio $\frac{r}{M}$, and increasing this ratio reduces the average solution error. This highlights a trade-off between computational load and learning performance: maintaining a low solution error for large $M$ requires proportionally increasing $r$, which comes at the cost of higher computational load on each device.
\end{remark}
\begin{remark}
\label{remark fix lr cong rate}
From Theorem~\ref{convergence performance CRA}, the convergence rate of the proposed method can be characterized as  
\begin{align}
    \label{cong rate fix lr}
\frac{1}{T + 1} \sum_{t = 0}^{T} \mathbb{E} \left[ \left\| \nabla F(\mathbf{x}^t) \right\|^2 \right] \leq O\left( \frac{1}{\rho_1\sqrt{T+1}} \right) + \varepsilon,
\end{align}
where $\varepsilon$ is a constant indicating the solution error.  When the values of \( d_1, \ldots, d_M \) do not vary significantly from each other, based on \( d_{\min}M \approx Nr \) and (\ref{rho def1}), we have
\begin{align}
    \label{appro rho1}
    {\rho _1} \approx \frac{1}{N} - 2\left( {1 - \frac{r}{M}} \right)\frac{{\sqrt {2C_\alpha ^2} }}{N}.
\end{align}
According to (\ref{cong rate fix lr}) and (\ref{appro rho1}), we can derive that increasing the redundancy in data allocation, i.e., increasing the value of $r$, leads to an accelerated convergence of the proposed method.

In addition, when the robustness of the RBA rules applied at the server is enhanced and the value of \( C_\alpha \) decreases, the convergence of the proposed method improves. This indicates that using enhanced RBA rules can accelerate the learning process of the proposed method.
\end{remark}
\begin{theorem}[Convergence performance of CRA-DL with decaying learning rates]
\label{convergence performance CRA decay}
Based on Assumptions~\ref{smooth assumption}-\ref{lower bound}, if ${C_\alpha } < \frac{{{d_{\min }}M}}{{2\sqrt 2 N\left( {r - \frac{{{r^2}}}{M}} \right)}}$, with decaying learning rates
\begin{align}
    \label{decaying lr}
    {\gamma ^t} = \frac{{{\rho _1} - \sqrt {\rho _1^2 - 4{\rho _2}\frac{{{\gamma ^0}{\rho _1} - {{\left( {{\gamma ^0}} \right)}^2}{\rho _2}}}{{\sqrt {t + 1} }}} }}{{2{\rho _2}}},
\end{align}
for ${\gamma ^0} < \frac{{{\rho _1}}}{{2{\rho _2}}}$, CRA-DL converges as
\begin{align}
\label{convergence decay lr}
  &{\min _{0 \leqslant t \leqslant T}}\mathbb{E}\left[ {{{\left\| {\nabla F\left( {{{\mathbf{x}}^t}} \right)} \right\|}^2}} \right] \nonumber \\
   \leqslant& \frac{{F\left( {{{\mathbf{x}}^0}} \right) - F^*}}{{\left[ {{\gamma ^0}{\rho _1} - {{\left( {{\gamma ^0}} \right)}^2}{\rho _2}} \right]\sqrt {T + 1} }} + \frac{{{\rho _3}}}{{{\rho _1} - {\gamma ^0}{\rho _2}}}\nonumber\\
   &+ \frac{{{{\left( {{\gamma ^0}} \right)}^2}{\rho _4}\left[ {2 + \log \left( {T + 1} \right)} \right]}}{{\left[ {{\gamma ^0}{\rho _1} - {{\left( {{\gamma ^0}} \right)}^2}{\rho _2}} \right]\sqrt {T + 1} }}.
\end{align}
\end{theorem}
\begin{proof}
Similar as the proof of Theorem~\ref{convergence performance CRA}, we can derive (\ref{average ite}). With the learning rates in (\ref{decaying lr}), we have
\begin{align}
    \label{decaying lr relation}
    {\gamma ^t}{\rho _1} - {\left( {{\gamma ^t}} \right)^2}{\rho _2} = \frac{{{\gamma ^0}{\rho _1} - {{\left( {{\gamma ^0}} \right)}^2}{\rho _2}}}{{\sqrt {t + 1} }},
\end{align}
where $\gamma^{t+1}<\gamma^{t}$ under the condition ${\gamma ^0} < \frac{{{\rho _1}}}{{2{\rho _2}}}$. From (\ref{average ite}) and (\ref{decaying lr relation}), we have
\begin{align}
    \label{min gradient}
  &{\min _{0 \leqslant t \leqslant T}}\mathbb{E}\left[ {{{\left\| {\nabla F\left( {{{\mathbf{x}}^t}} \right)} \right\|}^2}} \right]  \nonumber \\
   \leqslant &\frac{{\frac{1}{{T + 1}}\sum\limits_{t = 0}^T {\left( {{\gamma ^t}{\rho _1} - {{\left( {{\gamma ^t}} \right)}^2}{\rho _2}} \right)\mathbb{E}\left[ {{{\left\| {\nabla F\left( {{{\mathbf{x}}^t}} \right)} \right\|}^2}} \right]} }}{{\frac{1}{{T + 1}}\sum\limits_{t = 0}^T {\left( {{\gamma ^t}{\rho _1} - {{\left( {{\gamma ^t}} \right)}^2}{\rho _2}} \right)} }}  \nonumber \\
   \leqslant& \frac{{\frac{{F\left( {{{\mathbf{x}}^0}} \right) - F^*}}{{T + 1}} + \frac{1}{{T + 1}}\sum\limits_{t = 0}^T {\left[ {{\gamma ^t}{\rho _3} + {{\left( {{\gamma ^t}} \right)}^2}{\rho _4}} \right]} }}{{\frac{1}{{T + 1}}\sum\limits_{t = 0}^T {\left( {{\gamma ^t}{\rho _1} - {{\left( {{\gamma ^t}} \right)}^2}{\rho _2}} \right)} }}  \nonumber \\
      = &\frac{{F\left( {{{\mathbf{x}}^0}} \right) - F^*}}{{\sum\limits_{t = 0}^T {\frac{{{\gamma ^0}{\rho _1} - {{\left( {{\gamma ^0}} \right)}^2}{\rho _2}}}{{\sqrt {t + 1} }}} }} + \frac{{\sum\limits_{t = 0}^T {\left[ {{\gamma ^t}{\rho _3} + {{\left( {{\gamma ^t}} \right)}^2}{\rho _4}} \right]} }}{{\sum\limits_{t = 0}^T {\frac{{{\gamma ^0}{\rho _1} - {{\left( {{\gamma ^0}} \right)}^2}{\rho _2}}}{{\sqrt {t + 1} }}} }}  \nonumber \\
   \leqslant& \frac{{F\left( {{{\mathbf{x}}^0}} \right) - F^*}}{{\left[ {{\gamma ^0}{\rho _1} - {{\left( {{\gamma ^0}} \right)}^2}{\rho _2}} \right]\sqrt {T + 1} }} + \frac{{\sum\limits_{t = 0}^T {\left[ {{\gamma ^t}{\rho _3} + {{\left( {{\gamma ^t}} \right)}^2}{\rho _4}} \right]} }}{{\sum\limits_{t = 0}^T {\frac{{{\gamma ^0}{\rho _1} - {{\left( {{\gamma ^0}} \right)}^2}{\rho _2}}}{{\sqrt {t + 1} }}} }},
\end{align}
where the equality is obtained by substituting (\ref{decaying lr}) into (\ref{min gradient}), and the last equality is derived from the following inequality:
\begin{align}
    \label{basic ineq T1}
    \sum\limits_{t = 0}^T {\frac{1}{{\sqrt {t + 1} }} \geqslant \sqrt {T + 1} }.
\end{align}
 Based on (\ref{decaying lr relation}), we have
\begin{align}
    {\gamma ^t}{\rho _1} - {\left( {{\gamma ^t}} \right)^2}{\rho _2} = {\gamma ^t}\left( {{\rho _1} - {\gamma ^t}{\rho _2}} \right) = &\frac{{{\gamma ^0}{\rho _1} - {{\left( {{\gamma ^0}} \right)}^2}{\rho _2}}}{{\sqrt {t + 1} }} \nonumber\\
    \geqslant& {\gamma ^t}\left( {{\rho _1} - {\gamma ^0}{\rho _2}} \right),
\end{align}
which indicates $\frac{{{\gamma ^0}}}{{\sqrt {t + 1} }} \geqslant {\gamma ^t}$. 
Based on that, we can bound the second term on the right hand side of (\ref{min gradient}) as
\begin{align}
    \label{term 2}
   & \frac{{\sum\limits_{t = 0}^T {\left[ {{\gamma ^t}{\rho _3} + {{\left( {{\gamma ^t}} \right)}^2}{\rho _4}} \right]} }}{{\sum\limits_{t = 0}^T {\frac{{{\gamma ^0}{\rho _1} - {{\left( {{\gamma ^0}} \right)}^2}{\rho _2}}}{{\sqrt {t + 1} }}} }}\leqslant  \frac{{\sum\limits_{t = 0}^T {\left[ {\frac{{{\gamma ^0}}}{{\sqrt {t + 1} }}{\rho _3} + {{\left( {\frac{{{\gamma ^0}}}{{\sqrt {t + 1} }}} \right)}^2}{\rho _4}} \right]} }}{{\sum\limits_{t = 0}^T {\frac{{{\gamma ^0}{\rho _1} - {{\left( {{\gamma ^0}} \right)}^2}{\rho _2}}}{{\sqrt {t + 1} }}} }}  \nonumber \\
   =& \frac{{\sum\limits_{t = 0}^T {\left[ {\frac{{{\gamma ^0}}}{{\sqrt {t + 1} }}{\rho _3}} \right]} }}{{\sum\limits_{t = 0}^T {\frac{{{\gamma ^0}{\rho _1} - {{\left( {{\gamma ^0}} \right)}^2}{\rho _2}}}{{\sqrt {t + 1} }}} }} + \frac{{\sum\limits_{t = 0}^T {\left[ {\frac{{{{\left( {{\gamma ^0}} \right)}^2}}}{{t + 1}}{\rho _4}} \right]} }}{{\sum\limits_{t = 0}^T {\frac{{{\gamma ^0}{\rho _1} - {{\left( {{\gamma ^0}} \right)}^2}{\rho _2}}}{{\sqrt {t + 1} }}} }}  \nonumber \\
   \leqslant& \frac{{{\rho _3}}}{{{\rho _1} - {\gamma ^0}{\rho _2}}} + \frac{{{{\left( {{\gamma ^0}} \right)}^2}{\rho _4}\left[ {2 + \log \left( {T + 1} \right)} \right]}}{{\left[ {{\gamma ^0}{\rho _1} - {{\left( {{\gamma ^0}} \right)}^2}{\rho _2}} \right]\sqrt {T + 1} }},
\end{align}
where the last inequality is derived from (\ref{basic ineq T1}) and the inequality $\sum\limits_{t = 0}^T {\frac{1}{{t + 1}}}  \leqslant 2 + \log \left( {T + 1} \right)$.
Finally, substituting (\ref{term 2}) into (\ref{min gradient}) yields Theorem~\ref{convergence performance CRA decay}. 
\end{proof}
\begin{remark}
\label{remark decay lr}
In Theorem~\ref{convergence performance CRA decay}, on the right-hand side of (\ref{convergence decay lr}), when \( T \) approaches infinity, the asymptotic solution error approaches (\ref{error deacy}),
\begin{figure*}
\begin{align}
\label{error deacy}
    \frac{{{\rho _3}}}{{{\rho _1} - {\gamma ^0}{\rho _2}}} \approx \frac{{{\beta ^2}{M^2}\sqrt {2C_\alpha ^2} \left( {1 - \frac{r}{M}} \right)}}{{1 - 2\left( {1 - \frac{r}{M}} \right)\sqrt {2C_\alpha ^2}  - {\gamma ^0}\left\{ {\frac{{\left( {{\phi _1} - {\phi _2}} \right)2L}}{{{{\left( {1 - \alpha } \right)}^2}{N^2}}} + \frac{{{\phi _2}L}}{{{{\left( {1 - \alpha } \right)}^2}N}} + 8\frac{{LC_\alpha ^2}}{N}{{\left( {1 - \frac{r}{M}} \right)}^2}} \right\}}}.
\end{align}  
\end{figure*}
based on \( d_{\min}M \approx Nr \) when the values of \( d_1, \ldots, d_M \) do not vary significantly from each other. When \( r \) increases, i.e., the redundancy of data allocation increases, the asymptotic solution error provided in (\ref{error deacy}) diminishes. Specially, the asymptotic solution error reaches zero for \( r = M \). In addition, by decreasing the value of \( C_\alpha \), the asymptotic solution error decreases and better learning performance can be attained by CRA-DL with decaying learning rates.
\end{remark}
%\begin{remark}We would like to note that the numerical results in Section~\ref{simulations} will demonstrate that the asymptotic solution error in CRA-DL is indeed very close to zero in practice. In other words, CRA-DL can consistently converge to the optimal point with almost no solution error across various scenarios.\end{remark}
\begin{remark}
      In Theorem~\ref{convergence performance CRA} and Theorem~\ref{convergence performance CRA decay}, the bounds include the constant $C_\alpha$. Note that $C_\alpha$ depends on the fraction of Byzantine devices, i.e., $\alpha$, as observed from the values of \(C_\alpha^2\) for some commonly used RBA rules provided in Table I. Based on this, the bounds in Theorem~\ref{convergence performance CRA} and Theorem~\ref{convergence performance CRA decay} are both determined by the fraction of Byzantine devices. 
    
\end{remark}
\section{Numerical results}
\label{simulations}
In this section, we demonstrate the performance of the proposed method through numerical results on various tasks. We consider a commonly encountered Byzantine attack, namely Sign-flipping attack \cite{shi2022challenges}, which is state-of-the-art and has been used in much recent work such as \cite{zhu2023byzantine,yang2025enhanced}. For comparison, the following baseline methods are considered:

\begin{itemize} \item \textbf{Mean Averaging (MA):} The training data are allocated to the devices non-redundantly, and the server uses the mean aggregation rule to aggregate the local gradients from the honest devices and the disruptive messages from the Byzantine devices. \item \textbf{RBA-DL \cite{luan2024robust,yin2018byzantine,xie2018phocas}:} The training data are allocated to the devices non-redundantly. The server applies RBA rules to the local gradients from the honest devices and the disruptive messages from the Byzantine devices. \item \textbf{DL in the original SGC scheme (SGC-DL) \cite{bitar2020stochastic}:} The training data are allocated to the devices in a pair-wise balanced manner before training. During training iterations, the honest devices transmit coded gradients to the server, and the server aggregates the coded gradients from the honest devices and the disruptive information from the Byzantine devices using the mean aggregation rule. \item \textbf{Clairvoyant Method:} The training data are allocated to the devices non-redundantly. The server knows the identities of the devices in each iteration and only aggregates the local gradients from the honest devices using the mean aggregation rule. \end{itemize}
% \begin{enumerate}
%     \item \textbf{\textcolor{blue}{.}} Each Byzantine device transmits the true message multiplied by a negative coefficient. This coefficient is set to \(-2\) in our simulations.
%     \item \textbf{\textcolor{blue}{Gaussian attack \cite{shi2022challenges,zhu2023byzantine,odeyomi2025online}.}} Each Byzantine device transmits a vector of the same size as the true message, where the elements are randomly drawn from the Gaussian distribution \(\mathcal{N}(0,10000)\).
%     \item \textbf{\textcolor{blue}{Sample-duplicating attack \cite{wu2023byzantine,peng2022byzantine,odeyomi2025online}.}} Each Byzantine device randomly selects an honest device, duplicates the message of the selected honest device and sends this message to the server. 
% \end{enumerate}
\subsection{Linear Regression Task}
\label{linear regress sim}
We first consider a linear regression task with a synthetic dataset. In this task, the number of devices is \(N=100\) and the loss function can be expressed as
\begin{align}
\label{LR loss}
F\left( {\mathbf{x}} \right)= \sum\limits_{k = 1}^m {{f_k}\left( {\mathbf{x}} \right)},
{f_k}\left( {\mathbf{x}} \right) = \frac{1}{2}{\left( {\left\langle {{\mathbf{x}},{{\mathbf{z}}_k}} \right\rangle  - {y_k}} \right)^2},
\end{align}
where \(m = 1000\), \({{\mathbf{z}}_k} \in {\mathbb{R}^{100}}\), \(y_k \in {\mathbb{R}}\), \(k=1,...,1000\), and \(\mathbf{x} \in {\mathbb{R}^{100}}\). In this problem, the overall dataset \(\mathcal{D}\) consists of \(m=1000\) training data samples \(\left\{ {{{\mathbf{z}}_k},{y_k}} \right\}\), which are divided into 1000 subsets, each containing one data sample. Before the training starts, the training subsets are allocated uniformly and randomly to the devices so that each device obtains \(r\) subsets, which is an effective approximation of the pair-wise balanced allocation of the subsets. In (\ref{LR loss}), all the elements in \(\left\{ {{{\mathbf{z}}_1}, \dots, {{\mathbf{z}}_m}} \right\}\) are drawn independently from the normal distribution \(\mathcal{N}\left( {0,100} \right)\). To generate the values of \(y_k\), we first generate a random vector \(\mathbf{\overset{\lower0.5em\hbox{$\smash{\scriptscriptstyle\frown}$}}{x} }\) whose 100 elements are drawn from the standard normal distribution. 
To control the heterogeneity among the data subsets, \(y_k\) is generated as \(y_k \sim \mathcal{N}\left( \left\langle \mathbf{z}_k, \mathbf{\overset{\lower0.5em\hbox{$\smash{\scriptscriptstyle\frown}$}}{x}} + \mathbf{\tilde{x}}_k \right\rangle, 1 \right)\), for all \(k\), where \(\mathbf{\tilde{x}}_k \sim \mathcal{N}(0, k^2\sigma_H^2)\). A larger value of \(\sigma_H\) corresponds to a higher level of heterogeneity among the subsets.
Unless specified, the learning rate is fixed at $\gamma=0.001$ and we set $r=40$ and \(\sigma_H=0\) in our method. For the proposed method, we adopt four robust aggregation rules at the server that have been introduced in the literature, namely coordinate-wise median~\cite{yin2018byzantine}, trimmed mean~\cite{yin2018byzantine}, Phocas~\cite{xie2018phocas}, and MCA~\cite{luan2024robust}. Among these rules, coordinate-wise median, trimmed mean, and Phocas have been formally proven to be RBA rules \cite{dong2023byzantine}, whereas MCA has not yet been theoretically shown to satisfy the RBA properties. Based on this, the numerical results demonstrate that even robust aggregation rules that are not guaranteed to satisfy the RBA properties can be applied within the proposed method and still achieve improved Byzantine robustness.

\begin{figure}
    \centering
      \includegraphics[width=0.9\linewidth]{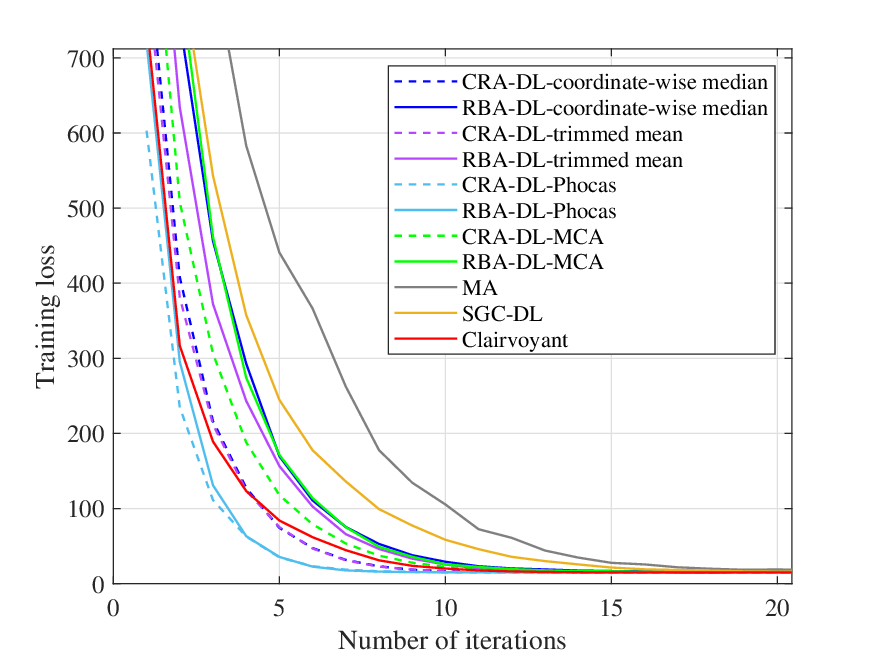}
    \caption{Training loss as a function of the number of iterations on the linear regression task for different methods with various RBA rules.}
    \label{fig: baseline}
\end{figure}

To compare the performance of the proposed method with the baseline methods, we plot the training loss as a function of the number of iterations for different methods in Fig.~\ref{fig: baseline}, where a number of RBA rules are applied for the proposed CRA-DL method and the baseline RBA-DL method. Here, we set $\alpha=0.2$. It can be observed that the learning performance of MA is easily influenced by the Byzantine attacks, which is because the disruptive information from Byzantine devices and the true information from honest devices are treated equally during aggregation. This aligns with our intuition. Among all the methods, the proposed method attains the best learning performance compared to the baseline methods. The advantage of our method over RBA-DL and SGC-DL lies in its ability to leverage the strengths of RBA rules and SGC simultaneously. This improves the robustness of the aggregation at the server by reducing the distance among messages sent by honest devices and more effectively mitigates the negative impact of disruptive information sent by Byzantine devices. It is worth noting that the clairvoyant method does not achieve the best learning performance, even though the server knows the identities of all devices. This is because the server discards the messages from Byzantine devices, meaning the data samples allocated to those devices, without allocation redundancy, cannot be utilized to update the global model. In contrast, our method allocates training data redundantly to devices. This allows our method not only to evade the disruptive information from Byzantine devices, preventing it from misleading the learning process, but also to compensate for the missing information from Byzantine devices with messages from honest devices based on data allocation redundancy.

To investigate the influence of data allocation redundancy on the learning performance of our method, we depict the training loss of CRA-DL as a function of the number of iterations under various values of $r$ in Fig.~\ref{fig: redundancy}. To clearly illustrate the trade-off, we also present the Pareto front between data allocation redundancy and training loss under a fixed number of iterations. Here, we consider different settings with various values of $\alpha$ and $\sigma_H$, where a larger value of $\alpha$ represents a higher fraction of Byzantine devices. A larger value of $\sigma_H$ indicates a higher level of heterogeneity among the training subsets. To be more specific, in Fig.~\ref{fig: redun02} and Fig.~\ref{fig: redun04}, we set $\sigma_H = 0$, whereas in Fig.~\ref{fig: hete r}, we set $\sigma_H = 0.001$. In both Fig.~\ref{fig: redun02} and Fig.~\ref{fig: hete r}, we use $\alpha = 0.2$, while in Fig.~\ref{fig: redun04}, we set $\alpha = 0.4$.
In Fig.~\ref{fig: redun02}, Fig.~\ref{fig: redun04}, and Fig.~\ref{fig: hete r}, the Pareto front is shown after 5 iterations, 10 iterations, and 30 iterations, respectively. 
In this scenario, the coordinate-wise median is adopted as the RBA rule. It can be observed that as the value of $r$ increases, indicating greater redundancy in data allocation, the learning performance of the proposed method improves, which is consistent with our theoretical analysis. This improvement and its alignment with our analysis are reflected in the following observations:
% \begin{figure}
%     \centering
%     \subfloat[]{%
%         \includegraphics[width=0.8\linewidth]{fig302.pdf}
%         \label{fig: redun02}
%     }

%     \subfloat[]{%
%         \includegraphics[width=0.8\linewidth]{fig304.pdf}
%         \label{fig: redun04}
%     }
    
%     \subfloat[]{%
%         \includegraphics[width=0.8\linewidth]{fig3hete.pdf}
%         \label{fig: hete r}
%     }
%     \caption{Training loss as a function of the number of iterations on the linear regression task for CRA-DL under various values of $r$. The subfigures depict the training loss as a function of $r$ under a fixed number of iterations. (a) $\alpha=0.2, \sigma_H=0$. (b) $\alpha=0.4, \sigma_H=0$. (c) $\alpha=0.2, \sigma_H=0.001$.}
%     \label{fig: redundancy}
% \end{figure}
\begin{figure*}[!t]
    \centering
    \subfloat[]{%
        \includegraphics[width=0.32\textwidth]{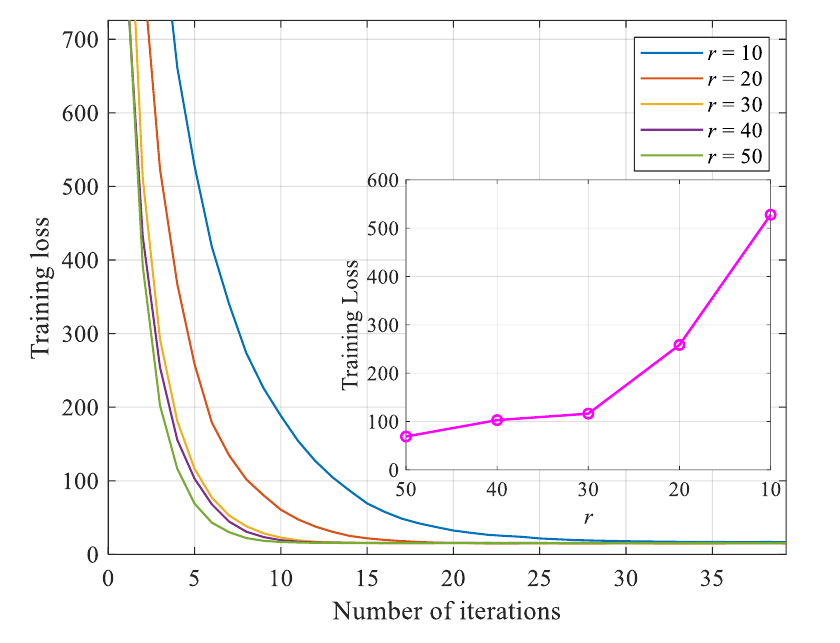}
        \label{fig: redun02}
    }
    \hfill
    \subfloat[]{%
        \includegraphics[width=0.32\textwidth]{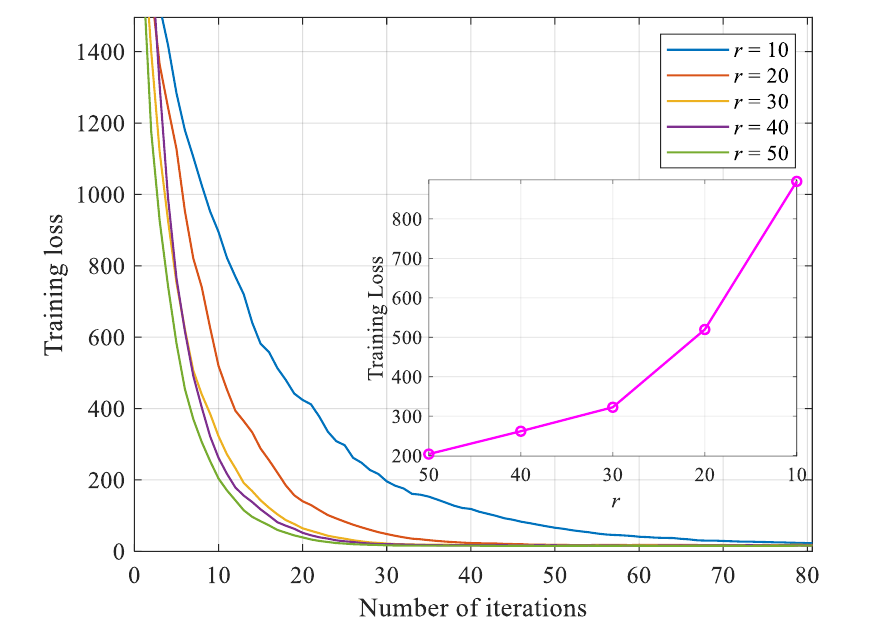}
        \label{fig: redun04}
    }
    \hfill
    \subfloat[]{%
        \includegraphics[width=0.32\textwidth]{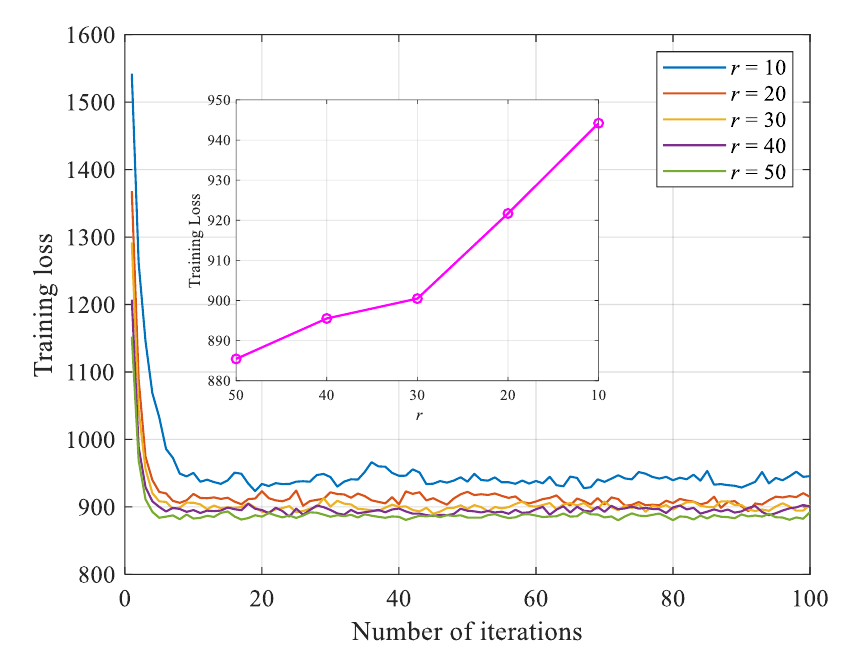}
        \label{fig: hete r}
    }
    \caption{Training loss as a function of the number of iterations on the linear regression task for CRA-DL under various values of $r$. The subfigures depict the training loss as a function of $r$ under a fixed number of iterations. (a) $\alpha=0.2, \sigma_H=0$. (b) $\alpha=0.4, \sigma_H=0$. (c) $\alpha=0.2, \sigma_H=0.001$.}
    \label{fig: redundancy}
\end{figure*}

\begin{enumerate}
\item In Fig.~\ref{fig: redundancy}, it can be observed that the proposed algorithm converges at a faster rate as the value of $r$ increases, which is consistent with Remark~\ref{remark fix lr cong rate}.
\item Under the settings of plotting Fig.~\ref{fig: redun02} and Fig.~\ref{fig: redun04}, the heterogeneity among the data subsets is very small, as all training subsets are drawn from the same distribution. According to our analysis in Theorem~\ref{convergence performance CRA} and Remark~\ref{remark fix lr}, when the heterogeneity among subsets is negligible, the solution error is also expected to be very small. For the considered linear regression problem, which is strongly convex and smooth, a very small solution error implies that the training loss can be reduced to near its minimum after a sufficiently large number of iterations, and vice versa.
In Fig.~\ref{fig: redun02} and Fig.~\ref{fig: redun04}, we can observe that a similar and very small training loss is attained under all settings after a large number of iterations. This confirms that the theoretical analysis presented in the paper is consistent with the numerical results.
\item In Fig.~\ref{fig: hete r}, where significant heterogeneity exists among the subsets, it can be observed that as the value of $r$ increases, a much lower training loss is achieved after a sufficiently large number of iterations. This implies that the solution error is reduced, which is consistent with the claims presented in Remark~\ref{remark fix lr} and Remark~\ref{remark fix lr cong rate} in Section~\ref{performance analysis}.
\end{enumerate}
These observations align with our intuition, which implies that improving learning performance and enhancing robustness against Byzantine attacks with an increasing value of $r$ come at the cost of increased computation and storage burdens on the devices. In practice, an appropriate trade-off between learning performance and the computation and storage burdens should be determined.

To investigate the learning performance of the proposed method under different levels of Byzantine attacks, we plot the training loss as a function of the number of iterations for CRA-DL under various values of $\alpha$ in Fig.~\ref{fig: no error}, where the coordinate-wise median is adopted as the RBA rule. It can be seen that as the value of $\alpha$ decreases, the learning performance of the proposed method improves. This aligns with our intuition, considering that better learning performance is expected with fewer Byzantine devices. When $\alpha = 0$, there is no Byzantine attack in the system. This case provides the optimal performance bound for learning under Byzantine attacks.
%From Fig.~\ref{fig: no error}, we can observe that the training loss under different values of $\alpha$ converges to the same value, demonstrating that the proposed method incurs negligible solution error even under Byzantine attacks. 
In addition, as shown in Fig.~\ref{fig: no error}, when $\alpha = 0.1$, the proposed method achieves nearly the same learning performance as in the case without any Byzantine attacks. This result indicates that the proposed method can almost completely mitigate the negative impact of Byzantine attacks when the number of Byzantine devices is modestly small.
It is worth noting that existing gradient coding methods designed to handle Byzantine attacks, such as \cite{hofmeister2024byzantine}, can also achieve the same learning performance, matching that of the case without Byzantine attacks. However, as pointed out in \cite{hofmeister2024byzantine}, the lower bound on the redundancy level in training data allocation required to achieve such robustness is $r = 110$ in this considered setting.
In contrast, the proposed method achieves comparable learning performance with a significantly lower redundancy level of $r = 40$. This substantial reduction in redundancy leads to much lower computational and storage burdens on the devices, highlighting the practical efficiency of the proposed approach under Byzantine attacks. The rationale behind the superiority of the proposed method is as follows. Given that machine learning algorithms are generally robust to noise, it is not necessary to fully recover the true global gradient in each iteration to update the global model, as is required in existing gradient coding methods designed to handle Byzantine attacks. Instead, it is often more efficient to obtain an approximate version of the true global gradient, as is done in the proposed method. As a result, the proposed method only requires a modest level of redundancy to achieve the same level of robustness to Byzantine attacks, thereby reducing the computational and storage burdens on the devices.

To demonstrate the robustness of the proposed method to the heterogeneity among subsets of the training data, we compare its performance with RBA-DL and plot the training loss as a function of the number of iterations in Fig.~\ref{fig: het} for both methods under various levels of heterogeneity, where the Sign-flipping attack is adopted with $\alpha=0.2$. Here, both methods use coordinate-wise median as the RBA rule for server-side aggregation. The key difference is that the proposed method aggregates coded gradients, whereas RBA-DL aggregates local gradients directly.
As shown in Fig.~\ref{fig: het}, the learning performance of the proposed method is significantly better than that of the baseline, especially under high heterogeneity. This is because, for the baseline under such scenarios, the local gradients sent by the honest devices differ significantly, which causes the global update obtained through mean aggregation in RBA-DL to deviate further from the true global gradient due to the influence of incorrect messages from Byzantine devices. In contrast, in the proposed method, the coded gradients sent by the honest devices remain close to each other, even under substantial heterogeneity. This improves the robustness of the aggregation at the server and results in a more accurate global model update, thereby enhancing the overall learning performance.

\subsection{Image Classification Task}
\label{image class}
We next consider an image classification task on the MNIST dataset~\cite{lecun1998mnist}, where \(N = 100\) devices collaboratively train a convolutional neural network consisting of two convolutional layers followed by two linear layers. The loss function used is the categorical cross-entropy loss, which is non-convex in this task. The training dataset comprises the MNIST training set, which contains 60{,}000 samples across 10 classes of handwritten digits.
To simulate a setting with zero heterogeneity among subsets, the training samples are randomly divided into \(M = 100\) subsets, with each subset containing the same number of samples. To simulate a setting with a high level of heterogeneity, the training samples are divided into \(M = 100\) subsets such that each subset contains samples from only a single class, also with equal size.
Before training, the subsets are allocated uniformly and randomly to the devices, so that each device holds \(r = 10\) subsets. For all experiments, the learning rate is set to \(\gamma = 0.01\). In the proposed method, the coordinate-wise median is used as the robust aggregation rule. In each iteration, a fraction \(\alpha = 0.2\) of the devices are assumed to be Byzantine.

\begin{figure}
    \centering
        \includegraphics[width=0.9\linewidth]{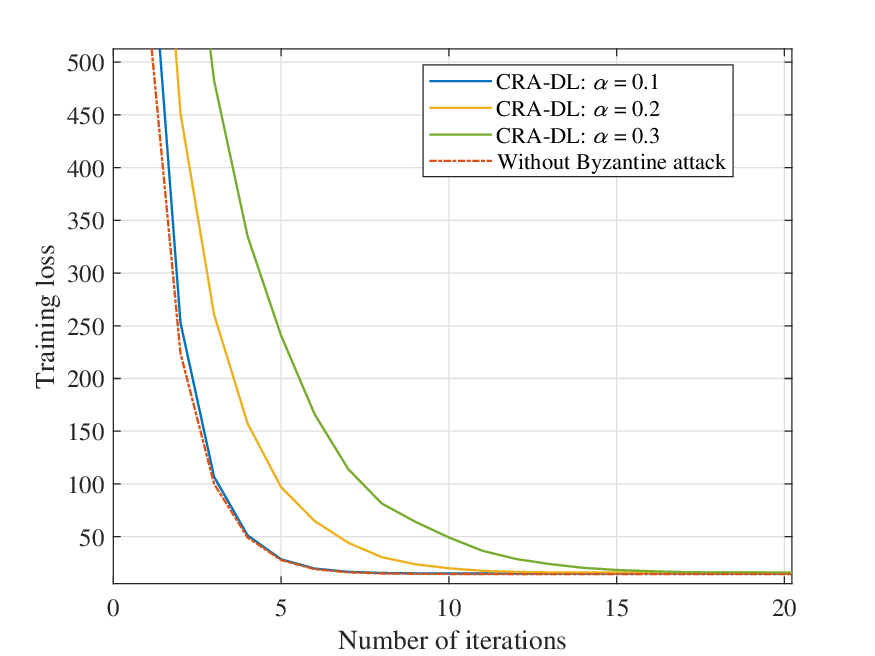}

    \caption{Training loss as a function of the number of iterations on the linear regression task for CRA-DL under various values of $\alpha$, where the case "Without Byzantine attack" corresponds to $\alpha=0$. {Note that, existing gradient coding methods designed to handle Byzantine attacks, such as \cite{hofmeister2024byzantine}, can also achieve the same learning performance, matching that of the case without Byzantine attacks.}}
    \label{fig: no error}
\end{figure}

\begin{figure}
    \centering
        \includegraphics[width=0.9\linewidth]{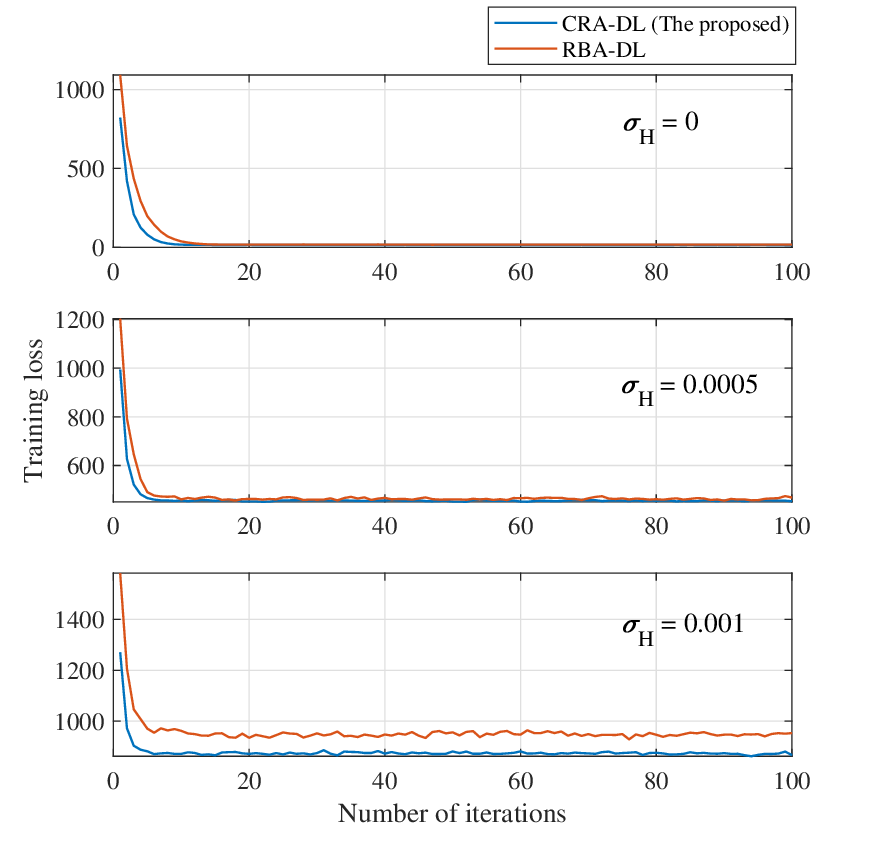}
     \caption{{Training loss as a function of the number of iterations on the linear regression task for CRA-DL and RBA-DL under various levels of heterogeneity among data subsets.}}
    \label{fig: het}
\end{figure}

To compare the performance of the proposed method with baseline methods, including MA, RBA-DL, and SGC-DL, we depict the training loss, training accuracy, test loss, and test accuracy as functions of the number of iterations for the different methods in Fig.~\ref{fig: iid mnist} and Fig.~\ref{fig: noniid mnist}, under varying levels of heterogeneity among the training data subsets. The same RBA rule is applied in both the proposed CRA-DL method and the baseline RBA-DL method for a fair comparison. In Fig.~\ref{fig: iid mnist}, there is no heterogeneity among the subsets, whereas in Fig.~\ref{fig: noniid mnist} significant heterogeneity is introduced. For better presentation of the results, the curves in Fig.~\ref{fig: noniid mnist} have been smoothed over a window of 150 iterations.
 
In Fig.~\ref{fig: iid mnist}, it can be observed that MA and SGC-DL perform the worst among all methods. This is because the mean aggregation at the server is significantly influenced by the incorrect information from the Byzantine devices, which aligns with our intuition and the results previously shown for the linear regression task in Fig.~\ref{fig: baseline}. Compared with MA and SGC-DL, RBA-DL achieves better performance, as the RBA rule at the server can partially filter out misleading information and thus mitigate its negative impact.
Among all methods, the proposed one achieves the best performance in both training and testing. This improvement stems not only from the robustness of the RBA rules at the server but also from the benefits of data allocation redundancy and gradient coding, which further enhance the robustness of the aggregation process.

In contrast to the observations in Fig.~\ref{fig: iid mnist}, Fig.~\ref{fig: noniid mnist} shows that when the heterogeneity among the subsets is significant, the performance of RBA-DL degrades substantially, even falling below that of MA and SGC-DL. This is because, under high heterogeneity, the local gradients sent by honest devices differ significantly, limiting the effectiveness of the RBA rule at the server. As a result, RBA-DL can no longer effectively filter out the incorrect messages from Byzantine devices.
In contrast, the proposed method achieves better performance by fully leveraging the benefits of gradient coding. The encoded gradients from honest devices become more similar, thereby enhancing the robustness of the aggregation process when RBA rules are applied at the server.
\begin{figure*}
    \centering
        \includegraphics[width=0.8\linewidth]{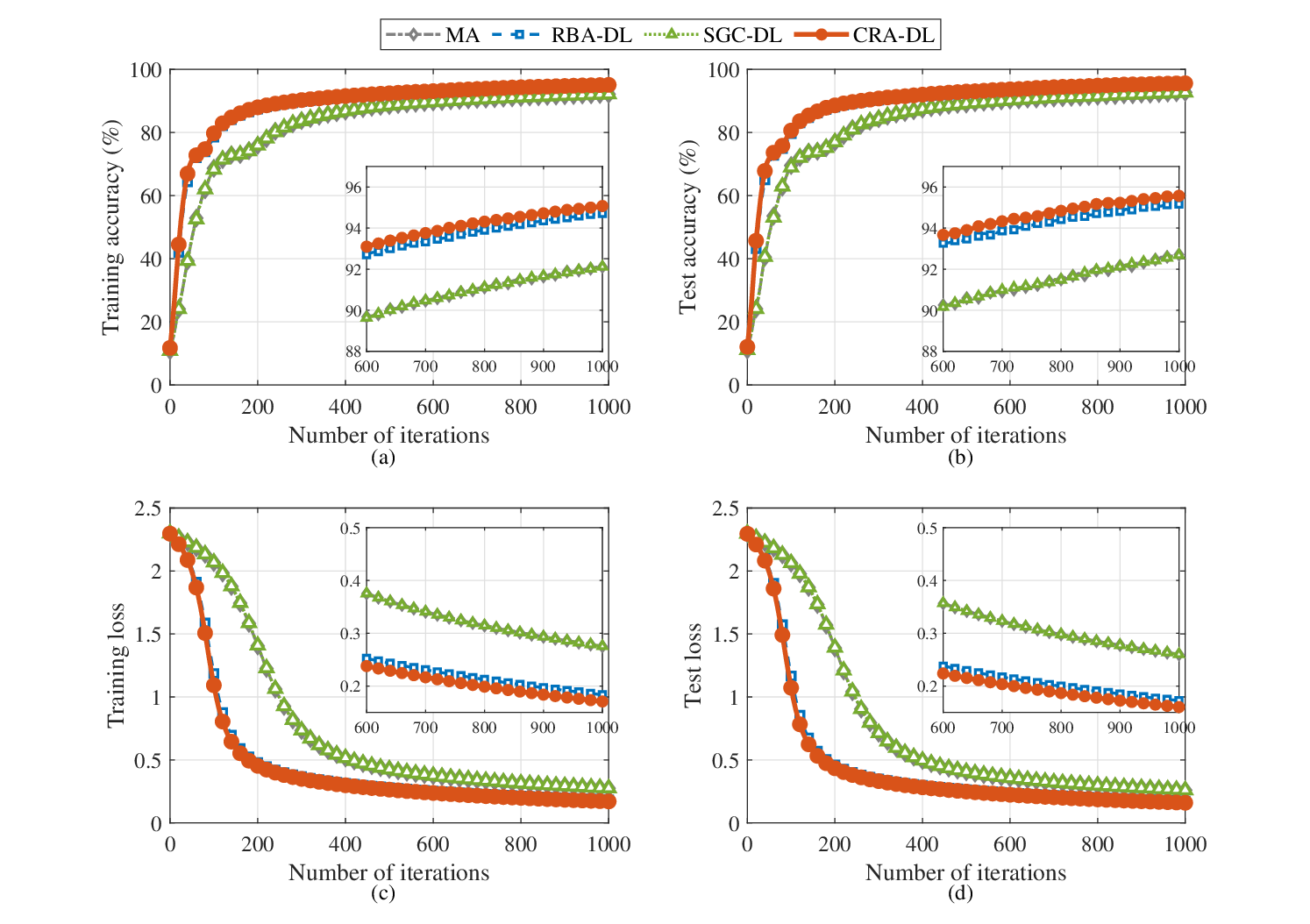}
    \caption{{{Training accuracy, test accuracy, training loss, and test loss as functions of the number of iterations for different methods on the image classification task without heterogeneity among data subsets. (a) Training accuracy. (b) Test accuracy. (c) Training loss. (d) Test loss.}}}
    \label{fig: iid mnist}
\end{figure*}

\begin{figure*}
    \centering
        \includegraphics[width=0.8\linewidth]{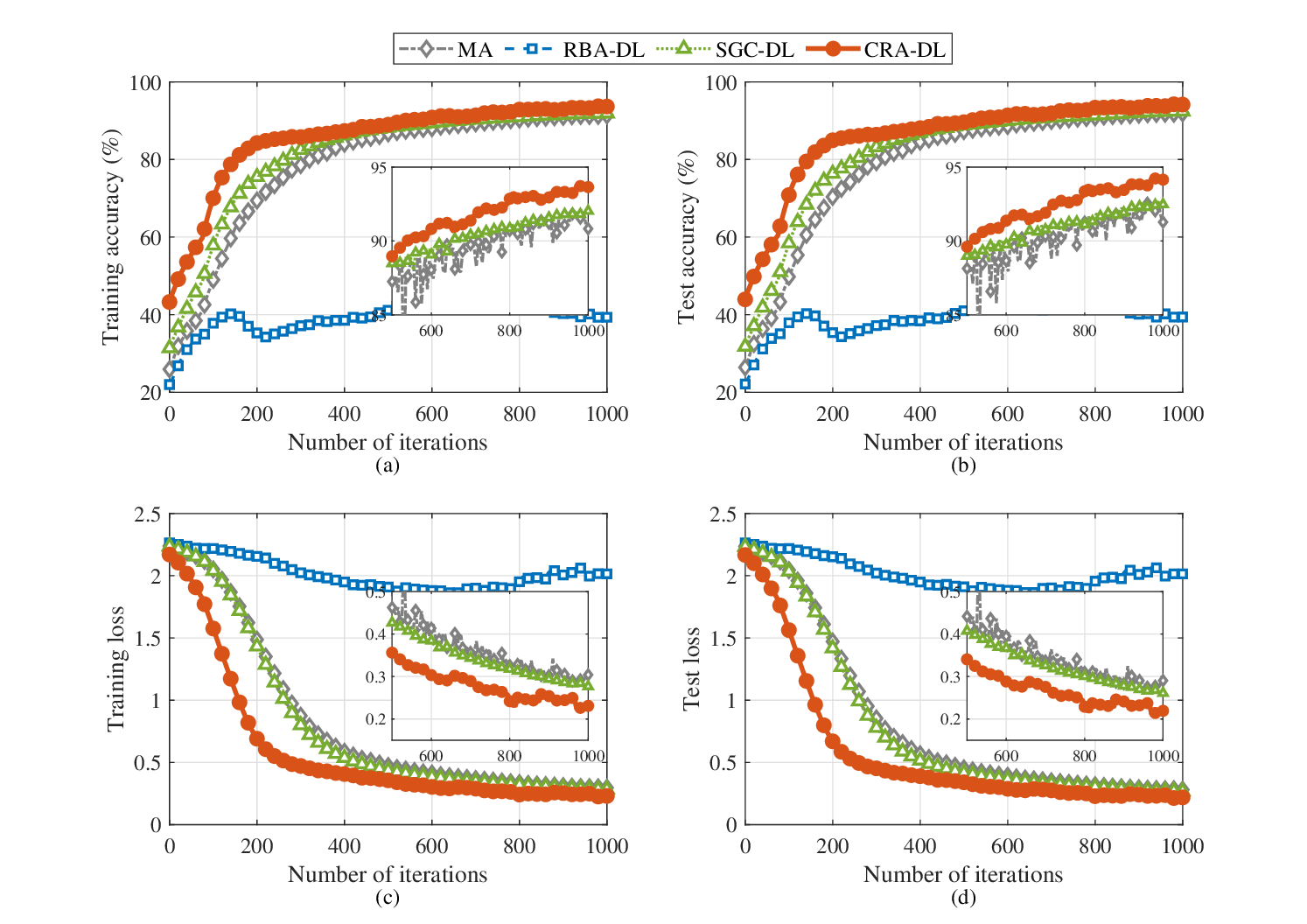}
    \caption{{{Training accuracy, test accuracy, training loss, and test loss as functions of the number of iterations for different methods on the image classification task with heterogeneity among data subsets. (a) Training accuracy. (b) Test accuracy. (c) Training loss. (d) Test loss.}}}
    \label{fig: noniid mnist}
\end{figure*}

\section{Conclusions}
\label{conclusions}
In this paper, the DL problem under Byzantine attacks was studied. To overcome the limitation of current DL methods applying RBA rules, we proposed CRA-DL. In the proposed method, before the training, the subsets of the training data are allocated to the devices in a pair-wise balanced manner. During training iterations, the server receives coded gradients from the honest devices and disruptive information from the Byzantine devices, and aggregates these information using RBA rules. By doing this, the global gradient is approximately recovered by the server to update the global model. The convergence performance of CRA-DL was analyzed and we provided numerical results to demonstrate the superiority of the proposed method compared to the baselines. 
 The proposed CRA-DL method has important practical implications for real-world DL systems. First, it enhances robustness to Byzantine attacks while requiring only a modest level of redundancy in data allocation, thereby reducing the computational and storage burdens on devices. This is an essential advantage for resource-constrained systems. Second, CRA-DL maintains strong robustness in the presence of data heterogeneity among subsets, making it suitable for applications such as healthcare, finance, and autonomous systems, where Byzantine resilience is critically needed under diverse operating conditions. Finally, as a meta-algorithm, CRA-DL is compatible with a wide range of existing RBA rules, enabling seamless integration into existing DL frameworks that already employ such rules.
In the future, we plan to extend the proposed method to scenarios where communication resources are highly limited, by compressing the communication between the devices and the server to mitigate the communication overhead.

\appendices
\section{Proof of Lemma \ref{max distance}}
\label{appendix lemma max distance}
Based on allocation of the training data in the pair-wise balanced manner, for $\forall i\neq j$, we have
\begin{align}
    \label{distance i and j}
  &{\mathbf{g}}_i^t - {\mathbf{g}}_j^t 
   = \sum\limits_{{k_1} \in \left\{ {\left. {{k_1}} \right|s\left( {i,{k_1}} \right) \ne 0} \right\}} {\frac{1}{{{d_{{k_1}}}}}\nabla {f_{{k_1}}}\left( {{{\mathbf{x}}^t}} \right)} \nonumber\\
   &- \sum\limits_{{k_2} \in \left\{ {\left. {{k_2}} \right|s\left( {j,{k_2}} \right) \ne 0} \right\}} {\frac{1}{{{d_{{k_2}}}}}\nabla {f_{{k_2}}}\left( {{{\mathbf{x}}^t}} \right)}  \nonumber \\
   =& \sum\limits_{{k_1} \in \left\{ {\left. {{k_1}} \right|s\left( {i,{k_1}} \right) \ne 0,s\left( {j,{k_1}} \right) = 0} \right\}} {\frac{1}{{{d_{{k_1}}}}}\nabla {f_{{k_1}}}\left( {{{\mathbf{x}}^t}} \right)} \nonumber\\
   &- \sum\limits_{{k_2} \in \left\{ {\left. {{k_2}} \right|s\left( {i,{k_2}} \right) = 0,s\left( {j,{k_2}} \right) \ne 0} \right\}} {\frac{1}{{{d_{{k_2}}}}}\nabla {f_{{k_2}}}\left( {{{\mathbf{x}}^t}} \right)},
\end{align}
according to (\ref{encoding}). Based on (\ref{distance i and j}), we can obtain (\ref{max distance i and j}), where the first three inequalities are derived from the following basic inequality:
\begin{align}
    \label{basic ineq 1}
    \left\| \sum_{i=1}^{n} \mathbf{a}_i \right\|^2 \leq n \sum_{i=1}^{n} \left\| \mathbf{a}_i \right\|^2, &\forall\mathbf{a}_i \in \mathbb{R}^D,
\end{align}
together with the fact that the set \({\left\{ {\left. {{k_1}} \right|s\left( {i,{k_1}} \right) \ne 0,s\left( {j,{k_1}} \right) = 0} \right\}}\) and the set \({\left\{ {\left. {{k_2}} \right|s\left( {i,{k_2}} \right) = 0,s\left( {j,{k_2}} \right) \ne 0} \right\}}\) both contain \({\left( {r - \frac{{{r^2}}}{M}} \right)}\) elements, the fourth inequality is derived from Assumption~\ref{assp bounded heter}, and the last inequality can be easily obtained according to the definition of $d_{min}$. This completes the proof. 
\begin{figure*}
\begin{align}
    \label{max distance i and j}
 & {\max _{i,j \in \left\{ {1,...,N} \right\}}}{\left\| {{\mathbf{g}}_i^t - {\mathbf{g}}_j^t} \right\|^2} \nonumber \\
   \leqslant & 2{\max _{i,j \in \left\{ {1,...,N} \right\}}}{\left\| {\sum\limits_{{k_1} \in \left\{ {\left. {{k_1}} \right|s\left( {i,{k_1}} \right) \ne 0,s\left( {j,{k_1}} \right) = 0} \right\}} {\frac{1}{{{d_{{k_1}}}}}\nabla {f_{{k_1}}}\left( {{{\mathbf{x}}^t}} \right)} } \right\|^2} + {\left\| {\sum\limits_{{k_2} \in \left\{ {\left. {{k_2}} \right|s\left( {i,{k_2}} \right) = 0,s\left( {j,{k_2}} \right) \ne 0} \right\}} {\frac{1}{{{d_{{k_2}}}}}\nabla {f_{{k_2}}}\left( {{{\mathbf{x}}^t}} \right)} } \right\|^2} \nonumber \\
   \leqslant & 2{\max _{i,j \in \left\{ {1,...,N} \right\}}}\left\{ {\left( {r - \frac{{{r^2}}}{M}} \right)\sum\limits_{{k_1} \in \left\{ {\left. {{k_1}} \right|s\left( {i,{k_1}} \right) \ne 0,s\left( {j,{k_1}} \right) = 0} \right\}} {{{\left\| {\frac{1}{{{d_{{k_1}}}}}\nabla {f_{{k_1}}}\left( {{{\mathbf{x}}^t}} \right) - \frac{1}{M}\frac{1}{{{d_{{k_1}}}}}\nabla F\left( {\mathbf{x}^t} \right) + \frac{1}{M}\frac{1}{{{d_{{k_1}}}}}\nabla F\left( {\mathbf{x}^t} \right)} \right\|}^2}} } \right. \nonumber \\
  & + \left. {\left( {r - \frac{{{r^2}}}{M}} \right)\sum\limits_{{k_2} \in \left\{ {\left. {{k_2}} \right|s\left( {i,{k_2}} \right) = 0,s\left( {j,{k_2}} \right) \ne 0} \right\}} {{{\left\| {\frac{1}{{{d_{{k_2}}}}}\nabla {f_{{k_2}}}\left( {{{\mathbf{x}}^t}} \right) - \frac{1}{M}\frac{1}{{{d_{{k_2}}}}}\nabla F\left( {\mathbf{x}^t} \right) + \frac{1}{M}\frac{1}{{{d_{{k_2}}}}}\nabla F\left( {\mathbf{x}^t} \right)} \right\|}^2}} } \right\}\nonumber\\
   \leqslant& 2{\max _{i,j \in \left\{ {1,...,N} \right\}}}\left\{ {2\left( {r - \frac{{{r^2}}}{M}} \right)\sum\limits_{{k_1} \in \left\{ {\left. {{k_1}} \right|s\left( {i,{k_1}} \right) \ne 0,s\left( {j,{k_1}} \right) = 0} \right\}} {{{\left\| {\frac{1}{{{d_{{k_1}}}}}\nabla {f_{{k_1}}}\left( {{{\mathbf{x}}^t}} \right) - \frac{1}{M}\frac{1}{{{d_{{k_1}}}}}\nabla F\left( {\mathbf{x}^t} \right)} \right\|}^2} + {{\left\| {\frac{1}{M}\frac{1}{{{d_{{k_1}}}}}\nabla F\left( {\mathbf{x}^t} \right)} \right\|}^2}} } \right. \nonumber \\
  & + \left. {2\left( {r - \frac{{{r^2}}}{M}} \right)\sum\limits_{{k_2} \in \left\{ {\left. {{k_2}} \right|s\left( {i,{k_2}} \right) = 0,s\left( {j,{k_2}} \right) \ne 0} \right\}} {{{\left\| {\frac{1}{{{d_{{k_2}}}}}\nabla {f_{{k_2}}}\left( {{{\mathbf{x}}^t}} \right) - \frac{1}{M}\frac{1}{{{d_{{k_2}}}}}\nabla F\left( {\mathbf{x}^t} \right)} \right\|}^2} + {{\left\| {\frac{1}{M}\frac{1}{{{d_{{k_2}}}}}\nabla F\left( {\mathbf{x}^t} \right)} \right\|}^2}} } \right\} \nonumber \\
   \leqslant &2{\max _{i,j \in \left\{ {1,...,N} \right\}}}\left\{ {2\left( {r - \frac{{{r^2}}}{M}} \right)\sum\limits_{{k_1} \in \left\{ {\left. {{k_1}} \right|s\left( {i,{k_1}} \right) \ne 0,s\left( {j,{k_1}} \right) = 0} \right\}} {\left( {\frac{1}{{d_{{k_1}}^2}}{\beta ^2} + \frac{1}{{d_{{k_1}}^2{M^2}}}{{\left\| {\nabla F\left( {\mathbf{x}^t} \right)} \right\|}^2}} \right)} } \right. \nonumber \\
  & + \left. {2\left( {r - \frac{{{r^2}}}{M}} \right)\sum\limits_{{k_2} \in \left\{ {\left. {{k_2}} \right|s\left( {i,{k_2}} \right) = 0,s\left( {j,{k_2}} \right) \ne 0} \right\}} {\left( {\frac{1}{{d_{{k_2}}^2}}{\beta ^2} + \frac{1}{{d_{{k_2}}^2{M^2}}}{{\left\| {\nabla F\left( {\mathbf{x}^t} \right)} \right\|}^2}} \right)} } \right\} \nonumber \\
   \leqslant& 8\frac{1}{{d_{\min }^2}}{\left( {r - \frac{{{r^2}}}{M}} \right)^2}\left( {{\beta ^2} + \frac{1}{{{M^2}}}{{\left\| {\nabla F\left( {\mathbf{x}^t} \right)} \right\|}^2}} \right).
\end{align}
\end{figure*}
\section{Proof of Lemma~\ref{honest average}}
\label{appendix lemma honest average}
Let us define
\begin{align}
    \label{a t}
    {{\mathbf{A}}^t} \triangleq \left[ {\frac{1}{{{d_1}}}\nabla {f_1}\left( {{{\mathbf{x}}^t}} \right),\frac{1}{{{d_2}}}\nabla {f_2}\left( {{{\mathbf{x}}^t}} \right),...,\frac{1}{{{d_M}}}\nabla {f_M}\left( {{{\mathbf{x}}^t}} \right)} \right],
\end{align}
and
\begin{align}
    \label{G t}
    {{\mathbf{G}}^t}  \triangleq  {{\mathbf{A}}^t}{{\mathbf{S}}^T} = \left[ {{\mathbf{g}}_1^t,...,{\mathbf{g}}_N^t} \right],
\end{align}
which are two matrices of size $D\times M$ and size $D\times N$, respectively. In (\ref{G t}), ${{\mathbf{S}}^T}$ is the transpose of the data allocation matrix $\mathbf{S}$. In addition, we define ${{\mathbf{h}}^t}$, an $N\times 1$ vector, to indicate the identities of the devices in iteration $t$, where the $i$-th element being 1 implies device $i$ is honest in iteration $t$, and the the $i$-th element being 0 implies the opposite. 

Based on the above definitions, we can express
\begin{align}
    \label{honest g}
    {{\mathbf{\bar g}}^t} = \frac{1}{{\left| {{\mathcal{H}^t}} \right|}}\sum\limits_{i \in {\mathcal{H}^t}} {{\mathbf{g}}_i^t}  = {{\mathbf{G}}^t}{{\mathbf{h}}^t}\frac{1}{{\left( {1 - \alpha } \right)N}}.
\end{align}
From (\ref{honest g}), we have
\begin{align}
    \label{honest g norm}
  {\left\| {{{{\mathbf{\bar g}}}^t}} \right\|^2} &= \frac{1}{{{{\left( {1 - \alpha } \right)}^2}{N^2}}}{\left( {{{\mathbf{h}}^t}} \right)^T}{\left( {{{\mathbf{G}}^t}} \right)^T}{{\mathbf{G}}^t}{{\mathbf{h}}^t} \nonumber\\
%  &= \frac{1}{{{{\left( {1 - \alpha } \right)}^2}{N^2}}}Tr\left[ {{{\left( {{{\mathbf{h}}^t}} \right)}^T}{{\left( {{{\mathbf{G}}^t}} \right)}^T}{{\mathbf{G}}^t}{{\mathbf{h}}^t}} \right] \nonumber \\
  & = \frac{1}{{{{\left( {1 - \alpha } \right)}^2}{N^2}}}Tr\left[ {{{\mathbf{h}}^t}{{\left( {{{\mathbf{h}}^t}} \right)}^T}{{\left( {{{\mathbf{G}}^t}} \right)}^T}{{\mathbf{G}}^t}} \right],
\end{align}
where $Tr(\cdot)$ is the trace of a square matrix. According to (\ref{honest g norm}), we can derive 
\begin{align}
    \label{exp honest norm}
 & \mathbb{E}\left( {\left. {{{\left\| {{{{\mathbf{\bar g}}}^t}} \right\|}^2}} \right|{\mathcal{F}^t}} \right)\nonumber\\
  =& \frac{1}{{{{\left( {1 - \alpha } \right)}^2}{N^2}}}\mathbb{E}\left( {\left. {Tr\left[ {{{\mathbf{h}}^t}{{\left( {{{\mathbf{h}}^t}} \right)}^T}{{\left( {{{\mathbf{G}}^t}} \right)}^T}{{\mathbf{G}}^t}} \right]} \right|{\mathcal{F}^t}} \right) \nonumber \\
%   =& \frac{1}{{{{\left( {1 - \alpha } \right)}^2}{N^2}}}Tr\left[ {\mathbb{E}\left( {\left. {{{\mathbf{h}}^t}{{\left( {{{\mathbf{h}}^t}} \right)}^T}{{\left( {{{\mathbf{G}}^t}} \right)}^T}{{\mathbf{G}}^t}} \right|{\mathcal{F}^t}} \right)} \right] \nonumber \\
   =& \frac{1}{{{{\left( {1 - \alpha } \right)}^2}{N^2}}}Tr\left[ {\mathbb{E}\left[ {{{\mathbf{h}}^t}{{\left( {{{\mathbf{h}}^t}} \right)}^T}} \right]{{\left( {{{\mathbf{G}}^t}} \right)}^T}{{\mathbf{G}}^t}} \right].
\end{align}
In (\ref{exp honest norm}), we have
\begin{align}
    \label{exp hht}
    \mathbb{E}\left[ {{{\mathbf{h}}^t}{{\left( {{{\mathbf{h}}^t}} \right)}^T}} \right] = \left( {\phi_1 - \phi_2} \right){\mathbf{I}} + \phi_2{\mathbf{1}}{{\mathbf{1}}^T}.
\end{align}
This is derived from the fact that, in each iteration, a fraction $\alpha$ of the devices are Byzantine devices, which is totally random. From this perspective, for $i\neq j$, the probability of device $i$ and $j$ being honest devices is $\Pr \left( {h_i^t = 1,h_j^t = 1} \right)=\phi_2$
and the probability of device $i$ being honest is $\Pr \left( {h_i^t = 1} \right)=\phi_1$, where $\phi_1$ and $\phi_2$ are defined in (\ref{phi1 and 2}). 

Next, substituting (\ref{exp hht}) into (\ref{exp honest norm}), we have
\begin{align}
    \label{exp honest norm 1}
 & \mathbb{E}\left( {\left. {{{\left\| {{{{\mathbf{\bar g}}}^t}} \right\|}^2}} \right|{\mathcal{F}^t}} \right) \nonumber\\
  =& \frac{1}{{{{\left( {1 - \alpha } \right)}^2}{N^2}}}Tr\left[ {\left( {\left( {\phi_1-\phi_2} \right){\mathbf{I}} + \phi_2{\mathbf{1}}{{\mathbf{1}}^T}} \right){{\left( {{{\mathbf{G}}^t}} \right)}^T}{{\mathbf{G}}^t}} \right] \nonumber \\
   =& \frac{1}{{{{\left( {1 - \alpha } \right)}^2}{N^2}}}Tr\left[ {\left( {\phi_1-\phi_2} \right){\mathbf{I}}{{\left( {{{\mathbf{G}}^t}} \right)}^T}{{\mathbf{G}}^t}} \right] \nonumber\\
   &+ \frac{1}{{{{\left( {1 - \alpha } \right)}^2}{N^2}}}Tr\left[ {\phi_2{\mathbf{1}}{{\mathbf{1}}^T}{{\left( {{{\mathbf{G}}^t}} \right)}^T}{{\mathbf{G}}^t}} \right] \nonumber \\
   %=& \frac{{\left( {\phi_1-\phi_2} \right)}}{{{{\left( {1 - \alpha } \right)}^2}{N^2}}}Tr\left[ {{{\left( {{{\mathbf{G}}^t}} \right)}^T}{{\mathbf{G}}^t}} \right] \nonumber\\
   %&+ \frac{\phi_2}{{{{\left( {1 - \alpha } \right)}^2}{N^2}}}Tr\left[ {{\mathbf{1}}{{\mathbf{1}}^T}{{\left( {{{\mathbf{G}}^t}} \right)}^T}{{\mathbf{G}}^t}} \right] \nonumber \\
   =& \frac{{\left( {\phi_1-\phi_2} \right)}}{{{{\left( {1 - \alpha } \right)}^2}{N^2}}}\sum\limits_{i = 1}^N {{{\left\| {{\mathbf{g}}_i^t} \right\|}^2}}  + \frac{\phi_2}{{{{\left( {1 - \alpha } \right)}^2}{N^2}}}Tr\left[ {{{\mathbf{1}}^T}{{\left( {{{\mathbf{G}}^t}} \right)}^T}{{\mathbf{G}}^t}{\mathbf{1}}} \right] \nonumber \\
   %=& \frac{{\left( {\phi_1-\phi_2} \right)}}{{{{\left( {1 - \alpha } \right)}^2}{N^2}}}\sum\limits_{i = 1}^N {{{\left\| {{\mathbf{g}}_i^t} \right\|}^2}} + \frac{\phi_2}{{{{\left( {1 - \alpha } \right)}^2}{N^2}}}{{{\left\| {\sum\limits_{i = 1}^N {{\mathbf{g}}_i^t} } \right\|}^2}} \nonumber \\
   =& \frac{{\left( {\phi_1-\phi_2} \right)}}{{{{\left( {1 - \alpha } \right)}^2}{N^2}}}\sum\limits_{i = 1}^N {{{\left\| {{\mathbf{g}}_i^t} \right\|}^2}}  + \frac{\phi_2}{{{{\left( {1 - \alpha } \right)}^2}{N^2}}}{{{\left\| {\nabla F\left( {{{\mathbf{x}}^t}} \right)} \right\|}^2}},
\end{align}
based on (\ref{encoding}) and (\ref{G t}). In (\ref{exp honest norm 1}), we can derive the bound for $\sum\limits_{i = 1}^N {{{\left\| {{\mathbf{g}}_i^t} \right\|}^2}}$ in (\ref{bound sum}) by applying the basic inequality in (\ref{basic ineq 1}) and Assumption 2. 
\begin{figure*}
\begin{align}
    \label{bound sum}
 & \sum\limits_{i = 1}^N {{{\left\| {{\mathbf{g}}_i^t} \right\|}^2}}  = \sum\limits_{i = 1}^N {{{\left\| {\sum\limits_{k \in \left\{ {\left. k \right|s\left( {i,k} \right) \ne 0} \right\}} {\frac{1}{{{d_k}}}} \nabla {f_k}\left( {{{\mathbf{x}}^t}} \right)} \right\|}^2}}  \nonumber \\
   \leqslant& \frac{r}{{d_{\min }^2}}\sum\limits_{i = 1}^N {\sum\limits_{k \in \left\{ {\left. k \right|s\left( {i,k} \right) \ne 0} \right\}} {{{\left\| {\nabla {f_k}\left( {{{\mathbf{x}}^t}} \right)} \right\|}^2}} }  = \frac{r}{{d_{\min }^2}}\sum\limits_{i = 1}^N {\sum\limits_{k \in \left\{ {\left. k \right|s\left( {i,k} \right) \ne 0} \right\}} {{{\left\| {\nabla {f_k}\left( {{{\mathbf{x}}^t}} \right) - \frac{1}{M}\nabla F\left( {{{\mathbf{x}}^t}} \right) + \frac{1}{M}\nabla F\left( {{{\mathbf{x}}^t}} \right)} \right\|}^2}} }  \nonumber \\
   \leqslant& \frac{{2r}}{{d_{\min }^2}}\sum\limits_{i = 1}^N {\sum\limits_{k \in \left\{ {\left. k \right|s\left( {i,k} \right) \ne 0} \right\}} {\left\{ {{{\left\| {\nabla {f_k}\left( {{{\mathbf{x}}^t}} \right) - \frac{1}{M}\nabla F\left( {{{\mathbf{x}}^t}} \right)} \right\|}^2} + {{\left\| {\frac{1}{M}\nabla F\left( {{{\mathbf{x}}^t}} \right)} \right\|}^2}} \right\}} }  \nonumber \\
   \leqslant& \frac{{2r}}{{d_{\min }^2}}\sum\limits_{i = 1}^N {\sum\limits_{k \in \left\{ {\left. k \right|s\left( {i,k} \right) \ne 0} \right\}} {\left[ {{\beta ^2} + {{\left\| {\frac{1}{M}\nabla F\left( {{{\mathbf{x}}^t}} \right)} \right\|}^2}} \right]} }  = \frac{{2N{r^2}}}{{d_{\min }^2}}\left[ {{\beta ^2} + \frac{1}{{{M^2}}}{{\left\| {\nabla F\left( {{{\mathbf{x}}^t}} \right)} \right\|}^2}} \right].
\end{align}
\end{figure*}
Substituting (\ref{bound sum}) into (\ref{exp honest norm 1}), we have (\ref{bound honest average}), which completes the proof. 

\bibliographystyle{IEEEtran}   
\bibliography{reference}  

\begin{IEEEbiography}[{\includegraphics[width=1in,height=1.25in,clip,keepaspectratio]{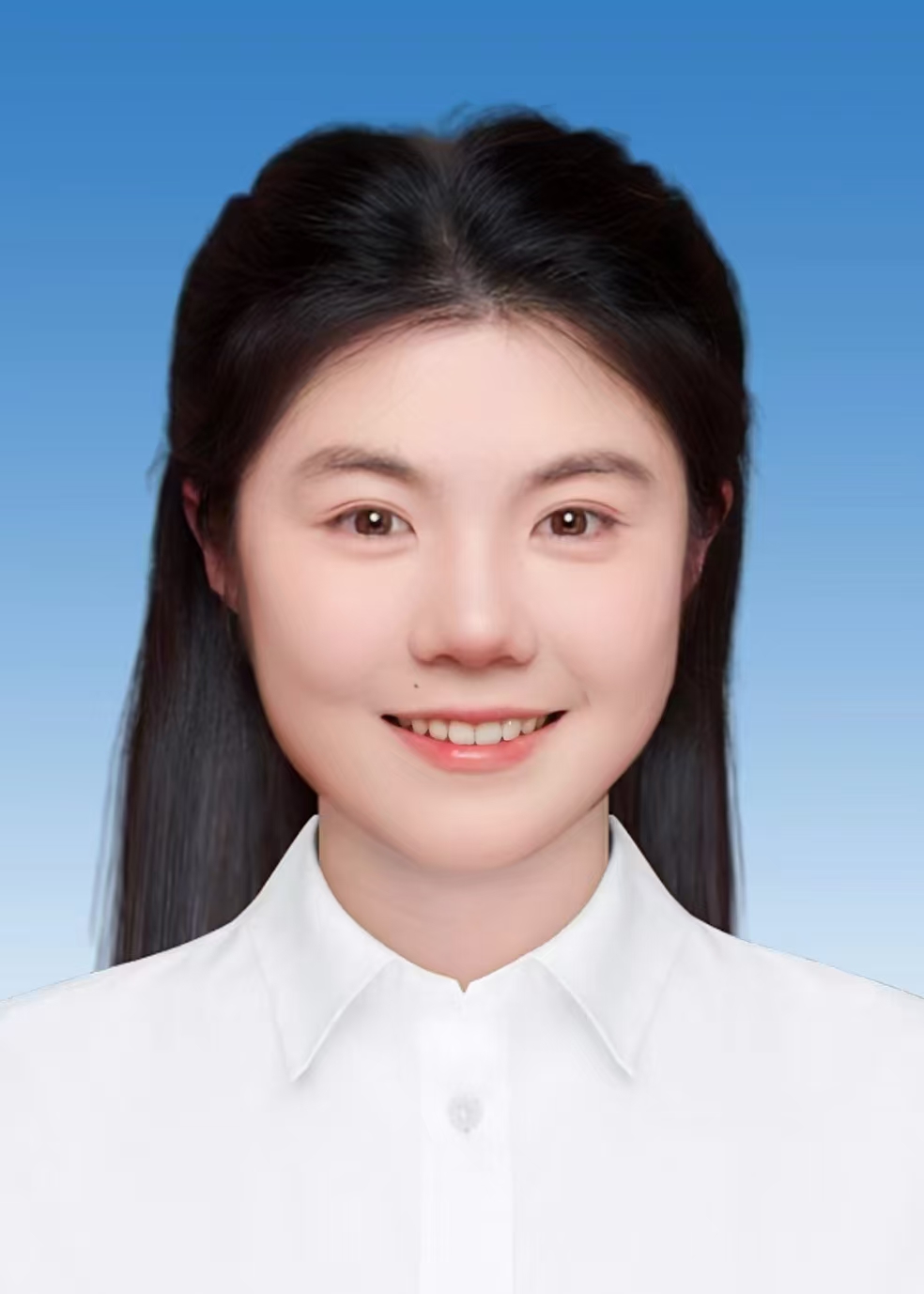}}]{Chengxi Li}
(Member, IEEE) received the B.S. degree in electronic engineering from the University of Electronic Science and Technology of China, Chengdu, China, in 2018, and the Ph.D. degree in electronic engineering from Tsinghua University, Beijing, China, in 2022. She is currently a Postdoctoral Researcher with the Division of Information Science and Engineering (ISE), KTH Royal Institute of Technology. Her research interests lie in distributed learning, federated learning, signal processing and information theory. She is a Digital Futures Postdoc Fellow and a Marie Skłodowska-Curie Actions (MSCA) Postdoc Fellow.
\end{IEEEbiography}

\begin{IEEEbiography}[{\includegraphics[width=1in,height=1.25in,clip,keepaspectratio]{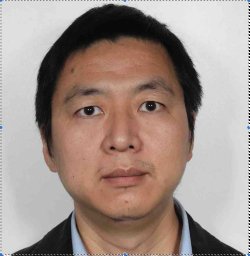}}]{Ming Xiao}
(Senior Member, IEEE) received the bachelor’s and master’s degrees in engineering from the University of Electronic Science and Technology of China, Chengdu, in 1997 and 2002, respectively, and the Ph.D. degree from the Chalmers University of Technology, Sweden, in November 2007. Since November 2007, he has been with the Department of Information Science and Engineering, School of Electrical Engineering and Computer Science, KTH Royal Institute of Technology, Sweden, where he is
currently a Professor. He received the IEEE Vehicular Technology Society Best Magazine Paper Award 2023. He was an Editor of IEEE TRANSACTIONS ON COMMUNICATIONS from 2012 to 2017, IEEE COMMUNICATIONS LETTERS (a Senior Editor) Since January 2015, and IEEE WIRELESS COMMUNICATIONS LETTERS from 2012 to 2016, and has
been an Editor of IEEE TRANSACTIONS ON WIRELESS COMMUNICATIONS since 2018. He has been an Area Editor of IEEE OPEN JOURNAL OF THE COMMUNICATION SOCIETY since 2019.
\end{IEEEbiography}

\begin{IEEEbiography}[{\includegraphics[width=1in,height=1.25in,clip,keepaspectratio]{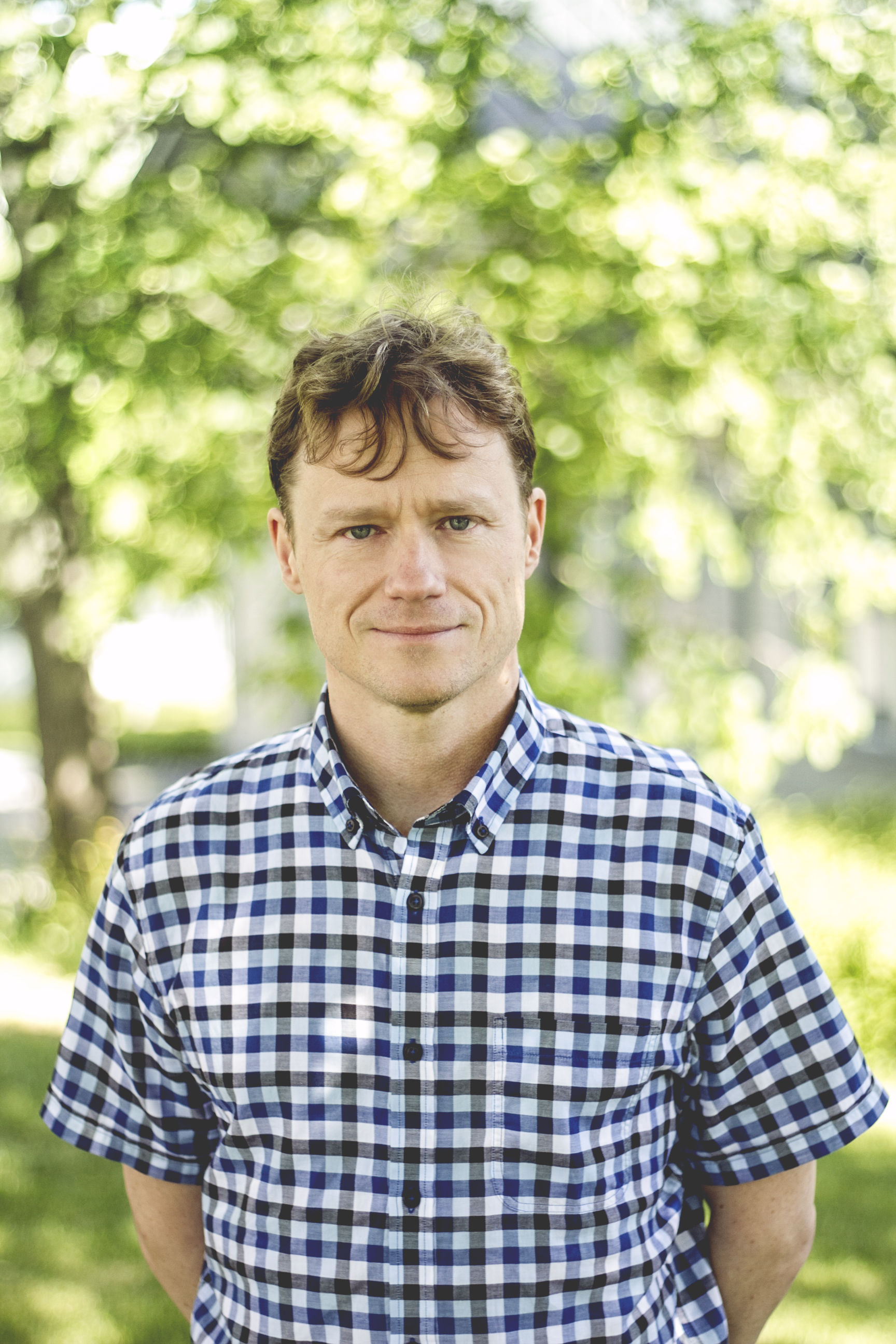}}]{Mikael Skoglund}
(Fellow, IEEE) received the Ph.D. degree from Chalmers University of Technology, Sweden, in 1997. In 1997, he joined the Royal Institute of Technology (KTH), Stockholm, Sweden, where he was appointed to the Chair in Communication Theory in 2003. At KTH, he heads the Division of Information Science and Engineering, as well as the Department of Intelligent Systems. Dr. Skoglund has worked on problems in source-channel coding, coding and transmission for wireless communications, Shannon theory, information-theoretic security, information theory for statistics and learning, information and control, and signal processing. He has authored and co-authored around 200 journal and more than 420 conference papers. From 2003 to 2008 he was an Associate Editor for IEEE TRANSACTIONS ON COMMUNICATIONS. In the interval 2008–2012 he was on the Editorial Board for IEEE TRANSACTIONS ON INFORMATION THEORY and starting in 2021 he joined it once again. He has served on numerous technical program committees for IEEE sponsored conferences, he was General Co-Chair for IEEE ITW 2019 and TPC Co-Chair for IEEE ISIT 2022. He is an Elected Member of the IEEE Information Theory Society Board of Governors.
\end{IEEEbiography}
\end{document}